\documentclass{article}

\usepackage{iclr2027_conference,times}
\iclrfinalcopy

\usepackage{amsmath}
\usepackage{graphicx} 
\usepackage{titletoc}
\usepackage{pgfplots}
\pgfplotsset{compat=1.18}
\newcommand\DoToC{%
  \startcontents
  \printcontents{}{1}{\hrulefill\vskip0pt}
  \vskip0pt \noindent\hrulefill
  }
\usepackage[utf8]{inputenc} % allow utf-8 input
\usepackage{booktabs}
\usepackage{enumerate}
\usepackage{subcaption}
\usepackage{bbding} 
\usepackage{enumitem}
\usepackage[T1]{fontenc}    % use 8-bit T1 fonts
\usepackage{hyperref}       % hyperlinks
\usepackage{url}            % simple URL typesetting
\usepackage{booktabs}       % professional-quality tables
\usepackage{amsfonts}       % blackboard math symbols
\usepackage{nicefrac}       % compact symbols for 1/2, etc.
\usepackage{microtype}      % microtypography
\usepackage{xcolor}    
\usepackage{algorithm}
\usepackage{algorithmic}
\usepackage{tikz}
\usetikzlibrary{arrows.meta,positioning,fit}

\usepackage{ulem}
\usepackage[table]{xcolor}
\usepackage{colortbl}
\usepackage{amsmath}
\usepackage{booktabs}
\usepackage{makecell}
\usepackage{algorithm}
\usepackage{amsfonts, amssymb}
\usepackage{algorithmic}
\usepackage{mathtools}
\usepackage{amsthm}
\usepackage[mathscr]{eucal}
\usepackage{bm}
\usepackage{graphicx}
\usepackage{multirow}

\usepackage{amsmath,lipsum}
\usepackage{amssymb} 
\usepackage{wrapfig}
\usepackage{ulem}
\usepackage{amsmath,amsfonts,bm}
\def\eqref#1{equation~\ref{#1}}
\def\1{\bm{1}}

\DeclareMathAlphabet{\mathsfit}{\encodingdefault}{\sfdefault}{m}{sl}
\SetMathAlphabet{\mathsfit}{bold}{\encodingdefault}{\sfdefault}{bx}{n}

\usepackage[english]{babel}
\newtheorem{proposition}{Proposition}
\newtheorem{theorem}{Theorem}

\newtheorem{corollary}{Corollary}[theorem]
\newtheorem{lemma}{Lemma}
\newtheorem{definition}{Definition}
\hypersetup{
    colorlinks=true,
    linkcolor=black,
    citecolor=teal,
    urlcolor=black
}
\newtheorem{assumption}{Assumption}
\usepackage{subfiles}
\usepackage{xr}
\definecolor{myblue}{HTML}{b2f0ff}
\definecolor{myblue2}{HTML}{cef5ff}
\definecolor{myblue3}{HTML}{e7faff}
\definecolor{revisiondarkgreen}{RGB}{0,100,0}

\usepackage{xcolor}

\usepackage[utf8]{inputenc} % allow utf-8 input
\usepackage[T1]{fontenc}    % use 8-bit T1 fonts
\usepackage{hyperref}       % hyperlinks
\usepackage{url}            % simple URL typesetting
\usepackage{booktabs}       % professional-quality tables
\usepackage{amsfonts}       % blackboard math symbols
\usepackage{nicefrac}       % compact symbols for 1/2, etc.
\usepackage{microtype}      % microtypography
\usepackage{xcolor}         % colors

\title{Hessian Rank Constraint for Learning Structure of Nonlinear Latent Variable Models}

\author{
\normalfont Zijian Li$^{1,4}$ \quad Ruichu Cai$^{2}$ \quad Feng Xie$^{3}$ \quad Xinshuai Dong$^{4}$ \quad Haoyue Dai$^{4}$\\[0.35em]
\normalfont Yuewen Sun$^{1,4}$ \quad Yujia Zheng$^{5}$ \quad Guangyi Chen$^{1,4}$ \quad Yingyao Hu$^{6}$ \quad Kun Zhang$^{1,4}$\\[0.65em]
\normalfont\small $^{1}$Mohamed bin Zayed University of Artificial Intelligence, Abu Dhabi, UAE\\
\normalfont\small $^{2}$Guangdong University of Technology, Guangzhou, China\\
\normalfont\small $^{3}$Beijing Technology and Business University, Beijing, China\\
\normalfont\small $^{4}$Carnegie Mellon University, Pittsburgh, PA, USA\\
\normalfont\small $^{5}$University of Illinois Urbana--Champaign, Champaign, IL, USA\\
\normalfont\small $^{6}$Johns Hopkins University, Baltimore, MD, USA
}
\date{}

\usepackage{etoolbox}\newcommand{\smallmathenvironment}[1]{\BeforeBeginEnvironment{#1}{\begingroup\small}\AfterEndEnvironment{#1}{\endgroup}}\smallmathenvironment{equation}\smallmathenvironment{equation*}\smallmathenvironment{align}\smallmathenvironment{align*}\smallmathenvironment{alignat}\smallmathenvironment{alignat*}\smallmathenvironment{flalign}\smallmathenvironment{flalign*}\smallmathenvironment{gather}\smallmathenvironment{gather*}\smallmathenvironment{multline}\smallmathenvironment{multline*}\smallmathenvironment{displaymath}
\begin{document}

\maketitle
\lhead{}

\begin{abstract}

Uncovering latent variables and their causal relations from observed data is a fundamental yet challenging problem. Existing methods often rely on restrictive assumptions, such as linear relations or invertible mixing functions. To better address this problem under general nonlinear mixing procedures, we propose a condition called the cross-Hessian Rank Constraint (HRC), which serves as a primitive rank-based tool for nonlinear latent causal discovery. In particular, we show that a rank-based property arises from the cross-Hessian of the observed-data log-density in the nonlinear case, revealing information about the latent variables, and reduces to the Tetrad constraints in the linear Gaussian case. More specifically, when two groups of observed variables are d-separated by a set of lower-dimensional latent variables, the rank of this cross-Hessian is equal to the dimension of the latent variables, under a mild affine derivative assumption on the conditional log-density derivatives. This assumption can be naturally satisfied when the noise level is low or the relevant nonlinearity is moderate.
\textcolor{black}{As a downstream application, we instantiate HRC in the pure one-factor measurement setting for locating latent variables and recovering their causal structure up to Markov equivalence.} Experimental results on synthetic and real-world datasets support the theoretical claims.
\end{abstract}

\section{Introduction}

Discovering causal structure in the presence of causally related latent confounders is one of the central and most challenging problems in causal discovery. Classical approaches, such as FCI and its variants \citep{spirtes2000causation,pearl2009causality,zhang2008completeness,colombo2012learning,akbari2021recursive,mokhtarian2025recursive}, use conditional independence information to recover aspects of the causal graph under latent confounding. However, these methods do not determine the number or the locations of latent variables, which has motivated a growing line of work on causal discovery among latent variables \citep{silva2006learning,kummerfeld2016causal}.

Recovering causal structure among latent variables is typically studied in settings where observed variables are leaves in the graph and are not directly adjacent. Within this setting, existing approaches rely on additional structural or distributional assumptions and exploit a variety of statistical signatures, including rank constraints \citep{silva2006learning,kummerfeld2016causal,huang2022latent,xie2022identification,dong2023versatile}, higher-order moments \citep{shimizu2009estimation,cai2019triad,xie2020generalized,adams2021identification,chen2022identification}, matrix decomposition \citep{anandkumar2013learning,chen2025identification}, copula-based formulations \citep{cui2018learning}, multi-domain information \citep{zeng2021causal,li2023causal,sturma2023unpaired,zhang2024causal}, and mixture-based oracles \citep{kivva2021learning,kivva2022identifiability}. Despite their methodological diversity, most of these methods still rely on restrictive assumptions, such as linear measurement relations from latent to observed variables or discrete latent variables. These assumptions may be hard to satisfy in real-world scenarios, where observed variables are often generated through general nonlinear processes. Please refer to Appendix \ref{app:related_works} for further discussion of related work on causal discovery in the presence of latent variables.

Extending causal discovery among latent variables to the general nonlinear setting remains relatively underexplored, and most of existing works focus on a special class of hierarchical latent-structure discovery problems. In this direction, \textcolor{black}{\citet{kong2023identification} identify hierarchical latent variables under a specific nonlinear latent variable model and then recover the structure among them.} However, they assume that latent variables are deterministic functions of measured variables. This assumption is often violated, for instance, if each measured variable has its own noise. More recently, \citet{prashant2025differentiable} study the nonlinear latent hierarchical model from a different perspective. Specifically, they leverage the rank of the Jacobian of the conditional expectation mapping between groups of observed variables to infer the size of separating latent sets and recover the latent structures. Nevertheless, they assume the existence of a function of the measured variables with the same dimensionality as the latent variables that make them conditionally independent. 
Although these works mark important steps, they primarily focus on a specific class of hierarchical latent-structure discovery problems and rely on restrictive assumptions. For causal discovery among latent variables under general nonlinear mixing processes, a more primitive tool is still urgently needed. 

Accordingly, in this paper, we propose a primitive tool for causal discovery among latent variables under general nonlinear mixing procedures. In linear latent variable models, Tetrad constraints characterize rank deficiencies in cross-covariance matrices \citep{huang2022latent}. Interestingly, we show that an analogous rank-based property also arises in the nonlinear setting. Specifically, if two groups of observed variables are d-separated by a lower-dimensional set of latent variables, then the cross-Hessian of the log-density of the observed variables has rank equal to the dimension of these latent variables, under a mild affine derivative assumption that holds for the conditional log-density derivatives. Furthermore, this assumption is compatible with the nonlinear mixing procedure from latent to observed variables and is naturally satisfied in regimes with low noise or moderate nonlinearity. Taken together, these observations lead to the cross-Hessian Rank Constraint (\textbf{HRC}), a primitive rank-based criterion for nonlinear latent causal discovery.

\textcolor{black}{As a concrete algorithmic application, we instantiate HRC in the standard pure one-factor measurement setting, where it can be used to locate latent variables and recover their causal structure up to Markov equivalence.} Furthermore, it can also be seamlessly incorporated with minimal modification into existing causal clustering methods for one-factor measurement models \citep{kummerfeld2016causal} and into PC-style rank-based causal discovery among latent variables \citep{silva2006learning,dong2023versatile}. Experiments on synthetic data support our theoretical results, and experiments on a real-world macroeconomic dataset show that the proposed approach recovers economically interpretable latent structure, highlighting its potential for practical nonlinear latent causal discovery.

\section{Preliminary}

Denote by $\mathcal G:=(\mathcal V,\mathcal E)$ a directed acyclic graph (DAG), where $\mathcal V:=\mathcal L\cup\mathcal X$ consists of $n+m$ nodes. $\mathcal L:=\{L_i\}_{i=1}^n$ denotes the set of $n$ latent nodes, and $\mathcal X:=\{X_j\}_{j=1}^m$ denotes the set of $m$ observed nodes.\footnote{Throughout the paper, we use ``observed variables'' and ``measured variables'' interchangeably to refer to variables directly available in the data.} The edge set $\mathcal E$ represents causal relations, and $\epsilon_{V_i}$ denotes the independent noise associated with node $V_i$. These variables are generated via the following nonlinear latent variable model:
\begin{equation}
\displaystyle
 L_i = f_{L_i}(\mathrm{Pa}(L_i), \epsilon_{L_i}), \quad
    X_j = f_{X_j}(\mathrm{Pa}(X_j), \epsilon_{X_j}),
\end{equation}
where $f_{L_i}$ and $f_{X_j}$ denote the general nonlinear functions and $\mathrm{Pa}(\cdot)$ denotes the set of parent variables of a given node in $\mathcal{G}$. Note that $\mathrm{Pa}(L_i)\subseteq\mathcal L$ and $\mathrm{Pa}(X_j)\subseteq\mathcal V$. Thus, latent variables are allowed to cause observed variables, and causal relations among observed variables are also allowed, but observed variables cannot be ancestors of latent variables. Such a data generation process provides a general formulation of the latent-variable causal models commonly considered in prior works \citep{silva2006learning,kummerfeld2016causal}, while allowing a general nonlinear process. 
\textcolor{black}{Uppercase letters denote random variables or vectors and lowercase letters their realizations; \(P\) and \(p\) denote distributions and densities, respectively. See Appendix~\ref{app:notation} for further notation.}
\textcolor{black}{For disjoint $A,B\subseteq\mathcal X$, let $p_{A,B}$ be their marginal density under $P_{A,B}$. Its log-density cross-Hessian is
\begin{equation}
\displaystyle
\mathbf H_{A,B}
:=
\frac{\partial^2\log p_{A,B}}
{\partial A\,\partial B^\top}.
\label{eq:cross-hessian}
\end{equation}
Here, derivatives are taken with respect to the corresponding arguments of the density. The notation $\operatorname{rank}(\mathbf H_{A,B})$ refers to the rank of this matrix-valued function.}
Based on the aforementioned data generation process, our goal is twofold. First, we seek a criterion from the observed distribution that reflects the d-separation by latent variables under general nonlinear processes. Second, we use this criterion for causal discovery among latent variables. More specifically, we identify the locations of latent variables, i.e., which groups of observed variables are associated with the same latent variables, and then recover the causal relations among the latent variables up to the Markov equivalence class.

\section{Cross-Hessian Rank Constraint}\label{sec:cross-hessian-rank-constraint}

In this section, we develop the cross-Hessian Rank Constraint (HRC). Specifically, under the affine derivative assumption, Theorem~\ref{thm:latent-dimension} establishes the rank upper bound implied by the d-separation of latent variables. Moreover, under an explicit non-degeneracy condition, Theorem~\ref{thm:cross-hessian-constraints} shows that this bound is tight and equals the minimum cardinality of these latent variables. \textcolor{black}{Finally, we quantify the deviation from the ideal low-rank form when the affine-derivative condition holds only approximately, and show how to estimate the cross-Hessian from observed variables using kernel density estimation (KDE), and provide a finite-confidence rank recovery condition.}

% 5/5:改成hessian rank dificiency
% \vspace{-1mm}
\subsection{Identifiability of the Dimensions of latent variables}
\begin{wrapfigure}{r}{0.33\textwidth}
\centering
\vspace{-35pt}
\scalebox{0.55}{
\begin{tikzpicture}[
    >={Stealth[length=4mm,width=2.8mm]},
    line width=1pt,
    latent/.style={
        draw,
        circle,
        minimum size=13mm,
        inner sep=0pt,
        font=\Large\bfseries
    },
    obs/.style={
        draw,
        circle,
        minimum size=13mm,
        inner sep=0pt,
        fill=gray!20,
        font=\Large\bfseries
    },
    groupbox/.style={
        draw,
        rounded corners,
        dashed,
        inner sep=7pt
    }
]

% latent variable
\node[latent] (L1) at (2.4,0) {$L_1$};

% observed variables
\node[obs] (X1) at (-0.6,-2.2) {$X_1$};
\node[obs] (X2) at ( 1.6,-2.2) {$X_2$};
\node[obs] (X3) at ( 3.2,-2.2) {$X_3$};
\node[obs] (X4) at ( 5.4,-2.2) {$X_4$};

% latent to observed
\draw[->] (L1) -- (X1);
\draw[->] (L1) -- (X2);
\draw[->] (L1) -- (X3);
\draw[->] (L1) -- (X4);

% within-group observed causal relations
\draw[->] (X1) -- (X2);
\draw[->] (X4) -- (X3);

% group boxes
\node[
    groupbox,
    fit=(X1)(X2),
    label={[font=\Large]below:{$A=\{X_1,X_2\}$}}
] {};

\node[
    groupbox,
    fit=(X3)(X4),
    label={[font=\Large]below:{$B=\{X_3,X_4\}$}}
] {};

\end{tikzpicture}
}
% \vspace{-5pt}
\caption{A causal graph where the latent variable $L_1$ d-separates two observed groups, $A=\{X_1,X_2\}$ and $B=\{X_3,X_4\}$}
\label{fig:latent-separation}
% \vspace{-20pt}
\end{wrapfigure}
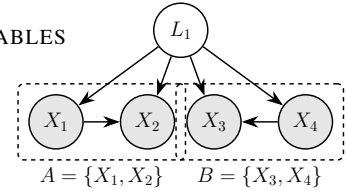
In the linear setting, d-separation by latent variables induces rank constraints on cross-covariance matrices, which can reveal latent dimension \citep{spearman1928pearson,silva2006learning,Sullivant-T-separation}. For instance, if a single latent variable d-separates two observed groups $A$ and $B$, then $\operatorname{rank}(\Sigma_{A,B})\leq 1$. Because covariance captures only linear associations, this method does not directly extend to nonlinear measurement processes. We establish an analogous rank constraint using the cross-Hessian of the observed marginal log-density. If a single latent variable d-separates $A$ and $B$, then $\operatorname{rank}(\mathbf H_{A,B})\leq 1$. When the minimum latent separator has cardinality one, the non-degeneracy condition introduced later sharpens this bound to equality. In Figure~\ref{fig:latent-separation}, for example, $L_1$ d-separates $A=\{X_1,X_2\}$ and $B=\{X_3,X_4\}$, so the upper bound is one and becomes exact under non-degeneracy. We first state two assumptions.

\begin{assumption}[\textbf{Smoothness and Integrability}]
\label{assump:regularity}
\textcolor{black}{
For fixed disjoint observed variables $A,B$ and latent variables
$\mathcal Z\subseteq\mathcal L$. let $Z$ denote its associated random vector.
The densities $p_{A,B,Z}$, $p_{A,B}$, $p_{A\mid Z}$, and $p_{B\mid Z}$ are
positive and twice continuously differentiable in their observed arguments,
and differentiation under the relevant integrals is valid up to second order.
Moreover, $\operatorname{Cov}(Z\mid A,B)$ exists and is finite $P_{A,B}$-almost surely.
}
\end{assumption}
\vspace{-1.5mm}
\textcolor{black}{
\textbf{Discussion:} These are standard regularity conditions for differentiating latent-variable marginal likelihoods \citep{louis1982finding}. Specifically, positivity and twice continuous differentiability ensure that the relevant log-densities and cross-Hessian are well defined. Differentiation under the integral justifies passing derivatives through the latent marginalization, and finiteness of the posterior covariance makes the covariance factorization well defined.
}
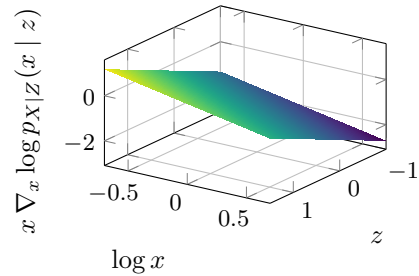
\begin{wrapfigure}{r}{0.40\textwidth}
\vspace{-0.5\baselineskip}
\centering
\begin{tikzpicture}
\begin{axis}[
    width=0.38\textwidth,
    height=0.30\textwidth,
    view={125}{32},
    xlabel={$z$},
    ylabel={$\log x$},
    zlabel={$x\,\nabla_x\log p_{X\mid Z}(x\mid z)$},
    domain=-1:1.5,
    y domain=-0.7:0.7,
    samples=24,
    samples y=20,
    colormap/viridis,
    grid=major,
    z buffer=sort,
    tick label style={font=\small},
    label style={font=\normalsize},
]
\addplot3[surf,shader=interp] {x-y-1};
\end{axis}
\end{tikzpicture}
\caption{\small Scaled first-order derivative for
$X=\exp(Z+\epsilon)$ with $\epsilon\sim\mathcal N(0,1)$.
Since $x\nabla_x\log p_{X\mid Z}(x\mid z)=z-\log x-1$, the surface
is a plane; the positive scaling by $x$ preserves affinity in $z$.}
\label{fig:main_nonlinear_density_local_linearity}
\vspace{-1.5\baselineskip}
\end{wrapfigure}
\vspace{-1.5mm}
\begin{assumption}[\textbf{Affine Derivatives}]
\label{assump:local-linearity}
\vspace{-2mm}
\textcolor{black}{
Let $A$ and $B$ be disjoint groups of observed variables, and fix a latent variables $\mathcal Z\subseteq\mathcal L$ between them in $\mathcal G$. Write
$d=|\mathcal Z|$, and let $Z\in\mathbb R^d$ be its associated random
vector. The first-order derivatives of the conditional log-densities are
affine in its realization $z$, which are shown as follows
\begin{equation}
\begin{split}
    &\nabla_a\log p_{A\mid Z}(a\mid z)
=
\mathbf u_A(a)+\mathbf V_A(a)z, \quad\\ 
&\nabla_b\log p_{B\mid Z}(b\mid z)
=
\mathbf u_B(b)+\mathbf V_B(b)z,
\end{split}
\label{eq:local-linearity-scores}
\end{equation}
in which $\mathbf u_A(a)\in\mathbb R^{|A|}$ and
$\mathbf u_B(b)\in\mathbb R^{|B|}$ are intercept terms, while
$\mathbf V_A(a)\in\mathbb R^{|A|\times d}$ and
$\mathbf V_B(b)\in\mathbb R^{|B|\times d}$ are slope matrices.
}
\end{assumption}
\textcolor{black}{
\textbf{Discussion:}
Let
$\mathbf s_A(a,z):=\nabla_a\log p_{A\mid Z}(a\mid z)$,
and define $\mathbf s_B(b,z)$ analogously. Assumption~\ref{assump:local-linearity}
does not require the mixing procedure to be
linear. Instead, it constrains how the conditional score changes with the
latent realization while the observed value is held fixed. More explicitly, the first-order Taylor expansion with respect to $z$ is}
\begin{equation}
    \mathbf s_A(a,z+\Delta z)
=
\mathbf s_A(a,z)
+
\nabla_z\mathbf s_A(a,z)\Delta z
+
\mathbf R_A(a,z,\Delta z),
\end{equation}
\textcolor{black}{where $\mathbf R_A(a,z,\Delta z)$ collects the second- and higher-order terms. Assumption~\ref{assump:local-linearity} requires
$\nabla_z\mathbf s_A(a,z)=\mathbf V_A(a)$ to be independent of $z$ and
the second-order derivatives with respect to $z$ to vanish. Thus, for any $z$ and $z+\Delta z$ in the domain where the assumption holds,
$\mathbf s_A(a,z+\Delta z)-\mathbf s_A(a,z)=\mathbf V_A(a)\Delta z$ exactly,
with no second- or higher-order remainder; equivalently,
$\mathbf R_A(a,z,\Delta z)=0$. Intuitively, $\mathbf V_A(a)$ describes how the conditional
score responds to movements in the latent space.
Figure~\ref{fig:main_nonlinear_density_local_linearity} illustrates this point using the nonlinear measurement
$X=\exp(\lambda^\top Z+\epsilon)$. Although the measurement function is nonlinear,
its conditional score varies affinely with $z$ and therefore satisfies the
assumption. More examples are provided in Appendix~\ref{app:local_linearity}. 
}

\begin{theorem}[\textbf{Cross-Hessian Rank under d-Separation by Latent Variables}]
\label{thm:latent-dimension}{
\textcolor{black}{
Let $A,B$ be disjoint observed groups and $\mathcal Z\subseteq\mathcal L$ a set of
latent variables that d-separates them. Let $Z\in\mathbb R^d$ be its associated
random vector, where $d=|\mathcal Z|$. If Assumptions~\ref{assump:regularity}
and~\ref{assump:local-linearity} hold for $Z$, then}
\begin{equation}
\operatorname{rank}\!\left(\mathbf H_{A,B}\right)\le d.
\label{eq:functional-rank-upper-bound}
\end{equation}
}
\end{theorem}
\textbf{Proof sketch and discussion.}
\textcolor{black}{
The full proof is in Appendix~\ref{app:the1}. Using the conditional scores
defined above, the causal Markov property and the fact that $\mathcal Z$
d-separates $A$ and $B$ give
$p_{A,B\mid Z}=p_{A\mid Z}p_{B\mid Z}$. Taking mixed derivatives of the marginal log-density gives
$\mathbf H_{A,B}=\operatorname{Cov}(\mathbf s_A(A,Z),\mathbf s_B(B,Z)\mid A,B)$.
Under Assumption~\ref{assump:local-linearity}, these conditional scores are
affine in $Z$, and hence}
\begin{equation}
\mathbf H_{A,B}
=
\mathbf V_A(A)\,
\operatorname{Cov}(Z\mid A,B)\,
\mathbf V_B(B)^\top.
\label{equ:hessian_decompose}
\end{equation}
\textcolor{black}{The three factors have dimensions $|A|\times d$,
$d\times d$, and $d\times |B|$, respectively. The rank inequality for
matrix products therefore gives $\operatorname{rank}(\mathbf H_{A,B})\leq d$.
It is noted that direct causal relations within $A$ or within $B$ are also allowed. Moreover, we next introduce a non-degeneracy assumption under which this upper bound is strengthened to equality as follows.}
\begin{assumption}[\textbf{Non-degeneracy}]
\label{assump:nondegeneracy}
For disjoint groups $A$ and $B$ of observed variables, let
$\mathcal Z^*$ be a minimum-cardinality set of latent variables that
d-separates them, with $d^*:=|\mathcal Z^*|$ and associated random vector
$Z^*\in\mathbb R^{d^*}$. Under $P_{A,B}$, almost surely,
(i) $\operatorname{Cov}(Z^*\mid A,B)$ is positive definite; and (ii)
$\mathbf V_A(A)$ and $\mathbf V_B(B)$ have full column rank $d^*$.
\end{assumption}
Assumption~\ref{assump:nondegeneracy} ensures that no rank is lost in the cross-Hessian decomposition associated with a minimum set of latent variables that d-separates $A$ and $B$. Condition~(i) requires the posterior covariance of $Z^*$ to retain variation along every latent direction. Condition~(ii) requires each observed group to be locally informative about every such direction. In a linear Gaussian measurement model, $\mathbf V_A=\Omega_A^{-1}\Lambda_A$ and $\mathbf V_B=\Omega_B^{-1}\Lambda_B$, where $\Omega_A$ and $\Omega_B$ denote the covariance matrices of the measurement noises in groups $A$ and $B$, respectively. Thus, condition~(ii) is equivalent to requiring both loading matrices to have full column rank. When full rank is feasible, violations form a measure-zero algebraic subset of the unrestricted loading-parameter space, paralleling the generic exclusion of rank-deficient parameterizations in linear rank-faithfulness results \citep{Sullivant-T-separation,huang2022latent,dong2023versatile}. Please refer to Appendix~
\ref{app:regular_non_degeneracy_condition} for more discussion.

\begin{theorem}[\textbf{Hessian Rank Constraint}]
\label{thm:cross-hessian-constraints}

Let $A$ and $B$ be disjoint groups of observed variables, and let
$\mathcal Z^*$ be a minimum-cardinality set of latent variables that d-separates them. Write $d^*:=|\mathcal Z^*|$. Under Assumptions~\ref{assump:regularity} and~\ref{assump:local-linearity} with respect to $\mathcal Z^*$, and under Assumption~\ref{assump:nondegeneracy},
we have
\begin{equation}
\operatorname{rank}\!\left(\mathbf H_{A,B}\right)=d^*.
\label{eq:functional-rank-equality}
\end{equation}
\vspace{-7mm}
\end{theorem}
\textbf{Proof sketch and discussion.}
Equation~(\ref{equ:hessian_decompose}) gives the cross-Hessian decomposition associated with
the minimum separator $\mathcal Z^*$ of cardinality $d^*$.
Under Assumption~\ref{assump:nondegeneracy}, the posterior covariance is
positive definite and both slope matrices have full column rank $d^*$;
hence the decomposition has rank exactly $d^*$.
Please refer to Appendix~\ref{app:the2} for the complete proof.

\textcolor{black}{As a direct specialization of Theorem~\ref{thm:cross-hessian-constraints}, we next consider the linear Gaussian case as shown in Corollary \ref{coro1}. The linear Gaussian conditional log-density derivatives are affine in the latent variables, so Assumption~\ref{assump:local-linearity} holds automatically. Appendix~\ref{app:cor1} gives the proof.}\begin{corollary}
\label{coro1}
\textbf{(Linear Gaussian Special Case)}
Let $A$ and $B$ be disjoint observed groups. Let $\mathcal Z^*$ be a
minimum latent d-separator between them, write $d^*=|\mathcal Z^*|$, and
let $Z^*\in\mathbb R^{d^*}$ be its associated random vector. Suppose the
full structural causal model is linear Gaussian, the joint distribution of $(A,B)$ is non-degenerate, and
Assumptions~\ref{assump:regularity} and~\ref{assump:nondegeneracy} hold.
Then
\[
\operatorname{rank}\!\left(\mathbf H_{A,B}\right)=d^*.
\]
\end{corollary}

\subsection{Robustness Analysis of the Cross-Hessian Rank Constraint}
\textcolor{black}{Theorems~\ref{thm:latent-dimension} and~\ref{thm:cross-hessian-constraints}
give exact cross-Hessian rank results under Assumption~\ref{assump:local-linearity}.
When the conditional scores are not exactly affine in the latent variables,
we quantify the resulting approximation error under a local smoothness condition.}

\textcolor{black}{\begin{theorem}[\textbf{Approximate Cross-Hessian Factorization}]
\label{the:error_bound}
For $A,B,Z$ as in Theorem~\ref{thm:latent-dimension}, fix
$y=(y_A^\top,y_B^\top)^\top$ with $p_{A,B}(y)>0$. Under
Assumption~\ref{assump:regularity}, suppose that, for some
$z_0\in\mathbb R^d$ and $\delta>0$:
\begin{itemize}[leftmargin=*,itemsep=0pt,topsep=1pt,parsep=0pt]
\item \textbf{(Lipschitz Continuity.)} The conditional scores
$\mathbf s_A(y_A,\cdot)$ and $\mathbf s_B(y_B,\cdot)$ are twice continuously
differentiable on $\mathcal B(z_0,\delta)$, and their componentwise Hessians
are $M$-Lipschitz on this ball.
\item \textbf{(Posterior Concentration.)} The conditional distribution
$P_{Z\mid A=y_A,B=y_B}$ is supported on $\mathcal B(z_0,\delta)$.
\end{itemize}
Let $\mathbf V_A=\nabla_z\mathbf s_A(y_A,z_0)$ and
$\mathbf V_B=\nabla_z\mathbf s_B(y_B,z_0)$ be the corresponding score
Jacobians. Let $\widetilde{\mathbf C}_A$ and $\widetilde{\mathbf C}_B$ collect
the componentwise curvature bounds $\|\nabla_z^2s_i(\cdot,z_0)\|_2+\delta M/3$
for the respective conditional scores.
Let $\mathbf H^{\mathrm{aff}}_{A,B}(y)$ denote the cross-Hessian under
Assumption~\ref{assump:local-linearity}. Then
{\small
\begin{equation}
\left\|\mathbf H_{A,B}(y)-\mathbf H^{\mathrm{aff}}_{A,B}(y)\right\|_2
\leq \tfrac12\delta^3\!\left(\|\mathbf V_A\|_2\|\widetilde{\mathbf C}_B\|_2+\|\mathbf V_B\|_2\|\widetilde{\mathbf C}_A\|_2\right)+\tfrac14\delta^4\|\widetilde{\mathbf C}_A\|_2\|\widetilde{\mathbf C}_B\|_2.
\label{eq:approximate-affine-cross-hessian-bound}
\end{equation}
}
\end{theorem}}
\textbf{Discussion.}\textcolor{black}{The proof is provided in Appendix~\ref{app:the3}. Theorem~\ref{the:error_bound} quantifies how departures from Assumption~\ref{assump:local-linearity} affect the cross-Hessian. The bound decreases when the posterior latent values are more concentrated and the conditional scores have lower local curvature. Thus, the theorem extends the exact affine characterization by explicitly controlling its approximation error. Although it does not prescribe how to estimate the cross-Hessian, it clarifies how the distribution of the latent variables and the curvature of the conditional scores affect HRC. Additional experiments in Appendix~\ref{app:more_exp} examine this behavior empirically.}

\subsection{Hessian Estimation and Rank Recovery}
\textcolor{black}{We estimate the cross-Hessian rank from observed data in two steps. First, we estimate the cross-Hessian matrix. Second, we recover its rank by thresholding the singular values of the estimated matrix. For the first step, write $Y=(A^\top,B^\top)^\top\in\mathbb R^s$, where
$s=|A|+|B|$, and let $Y^{(1)},\ldots,Y^{(N)}$ be the corresponding i.i.d.
marginal samples. Using kernel density estimation (KDE)
\citep{chen2017tutorial} with kernel $\mathcal K$ and bandwidth $h>0$, we estimate $p_{A,B}$ and its cross-Hessian by
{\small
\begin{equation}
\resizebox{\linewidth}{!}{$\displaystyle
\widehat{\mathbf H}_{A,B}(y)=\left[\nabla_y^2\log\widehat p_h(y)\right]_{A,B}=\frac{\nabla^2_{y_A,y_B}\widehat p_h(y)}{\widehat p_h(y)}-\frac{\nabla_{y_A}\widehat p_h(y)\,\nabla_{y_B}\widehat p_h(y)^\top}{\widehat p_h(y)^2},\quad
\widehat p_h(y)=\frac{1}{Nh^s}\sum_{i=1}^N\mathcal K\!\left(\frac{y-Y^{(i)}}{h}\right).
$}
\label{eq:kde-cross-hessian}
\end{equation}
}}
\textcolor{black}{The derivation of the estimated cross-Hessian in Equation~(\ref{eq:kde-cross-hessian}) is provided in Appendix~\ref{app:kde-cross-hessian-derivation}.}

\textcolor{black}{For the second step, finite-sample estimation error can make singular values that are zero for the true cross-Hessian appear slightly positive. We therefore compute the singular values of $\widehat{\mathbf H}_{A,B}(y)$, retain only those larger than a threshold $\tau$, and use their number as the estimated rank.}
\textcolor{black}{The theoretically valid threshold depends on unknown density and smoothness quantities. We therefore use a nonparametric bootstrap to calibrate a numerical threshold in practice. For each bootstrap replicate $b=1,\ldots,B_{\mathrm{boot}}$, we resample
$\{Y^{(i)}\}_{i=1}^N$ with replacement, recompute the marginal KDE cross-Hessian
$\widehat{\mathbf H}_{A,B}^{(b)}(y)$ using Equation~(\ref{eq:kde-cross-hessian}), and record
{\small
\begin{equation}
e_b
=
\left\|
\widehat{\mathbf H}_{A,B}^{(b)}(y)
-
\widehat{\mathbf H}_{A,B}(y)
\right\|_2.
\end{equation}
}
Let $\tau_{\mathrm{boot}}$ be the $q_{\mathrm{boot}}$-th empirical quantile of
$\{e_b\}_{b=1}^{B_{\mathrm{boot}}}$. Let $\sigma_1(\cdot)$ denote the largest singular value, and let $\rho\in(0,1)$ be a prespecified, dimensionless relative tolerance held fixed throughout a given analysis. We define the relative numerical threshold
$\tau_{\mathrm{rel}}
=
\rho\,\sigma_1(\widehat{\mathbf H}_{A,B}(y))$. This scale-adaptive threshold sets a lower bound on the final cutoff: singular values no larger than a $\rho$ fraction of the largest estimated singular value are treated as numerically zero even when the bootstrap threshold is smaller. We then report
{\small
\begin{equation}
\begin{aligned}
\tau
&=
\max\{\tau_{\mathrm{boot}},\tau_{\mathrm{rel}}\},&
\widehat r_{\tau}
&=
\#\left\{
j:
\sigma_j(\widehat{\mathbf H}_{A,B}(y))>\tau
\right\}.
\end{aligned}
\end{equation}
}
The resulting pointwise rank estimate counts every singular value above $\tau$ as nonzero and treats all remaining singular values as numerical zero.}
\textcolor{black}{The bootstrap operates at the matrix level: each resample produces a complete KDE cross-Hessian, and the empirical distribution of its spectral-norm deviation determines $\tau_{\mathrm{boot}}$. This provides a single data-adaptive cutoff that reflects estimation variability across the matrix, while $\tau_{\mathrm{rel}}$ prevents numerically negligible singular values from being retained when the bootstrap cutoff is small. Because HRC is evaluated pointwise, we repeat this procedure at a finite set of interior evaluation points after excluding low-density regions. This yields one rank estimate at each point; we use their modal value as the final rank estimate in all subsequent HRC decisions, reducing sensitivity to any individual location. Theoretical guarantees for the KDE cross-Hessian estimator and the resulting rank-recovery rule are provided in Appendix~\ref{app:kde-rank-recovery}, and implementation details for evaluation-point selection and rank aggregation are given in Appendix~\ref{app:rank-aggregation}.
}

% We now use the aforementioned results for causal discovery in nonlinear measurement models. Following the standard setting of one-factor measurement-model~\citep{kummerfeld2016causal}, we consider a pure measurement model in which each observed variable has at most one latent parent, has no observed parents, and has no correlated measurement error with other observed variables. We call an observed variable a pure indicator of $L_k$ if its only latent parent is $L_k$. Moreover, a set $C_k\subseteq X$ is called a pure cluster of $L_k$ if every variable in $C_k$ is a pure indicator of $L_k$; in this case, $L_k$ is called the \textit{key latent} of $C_k$. Under this setting, our application consists of two stages: first, we locate latent variables by clustering observed indicators that share the same one-dimensional key latent, replacing the classical covariance-rank constraint with the HRC rank-one constraint; second, given the estimated pure measurement model, we use pure indicators as surrogates to recover the causal structure among the latent variables up to Markov equivalence.

% 5.5谨慎提measurement model

% 我们用上面的工具学结构，具体来说我们focuse on
% \vspace{-1mm}
\section{Application in Learning the Structure among Latent Variables}\textcolor{black}{Although HRC is not restricted to pure measurement models, an end-to-end procedure requires observed groups that can be associated with individual latent variables. We therefore instantiate HRC in the standard pure one-factor measurement setting, using its rank constraints first to locate latent variables and then to test conditional independence among them. The resulting FOFC+PC procedure is an illustrative application rather than a restriction of HRC.}
We now use the aforementioned results for causal discovery in nonlinear measurement models \footnote{Please refer to Appendix \ref{def:measurement_model} for the definition of measurement model.}. Following the standard setting of one-factor measurement models~\citep{kummerfeld2016causal}, we consider a pure measurement model in which each observed variable has at most one latent parent, has no observed parents, and has no correlated measurement error with other observed variables. For a better understanding, the definition of pure indicator, pure cluster, and key latent is shown as:
\begin{definition}[\textbf{Pure Indicator, Pure Cluster, and Key Latent}]
\label{def:pure}
An observed variable $X_i$ is called a \textit{pure indicator} of a latent variable $L_k$ if, in the measurement part of the graph, $L_k$ is the only latent parent of $X_i$, $X_i$ has no observed parents, and $X_i$ shares no correlated measurement error with other observed variables. A set $C_k\subseteq\mathcal X$ is called a \textit{pure cluster} of $L_k$ if all variables in $C_k$ are pure indicators of $L_k$ and are not indicators of any other latent variable. In this case, $L_k$ is called the \textit{key latent} of $C_k$, since it is the common latent source underlying the cluster.
% \vspace{-2mm}
\end{definition}
Under this setting, our application consists of two stages: first, we locate latent variables by clustering observed indicators that share the same one-dimensional key latent, replacing the classical covariance-rank constraint with the HRC rank-one constraint; second, given the estimated pure measurement model, we use pure indicators as surrogates to recover the causal structure among the latent variables up to Markov equivalence. Note that our contribution here is not a new latent-location or causal-discovery algorithm, but the use of HRC as a new building block compatible with existing algorithms.

% Under this setting, our goal is not to introduce a new clustering or causal discovery pipeline from scratch, but to show that the proposed Hessian Rank Constraint can serve as a nonlinear rank-based tool that is compatible with existing latent-variable discovery algorithms. Specifically, we demonstrate two applications: first, HRC can replace the classical covariance-rank constraint in one-factor causal clustering to locate latent variables by grouping observed indicators that share the same one-dimensional key latent; second, once a pure measurement model is obtained, HRC can be used with pure indicators as surrogates to test conditional independencies among latent variables and recover their causal structure up to Markov equivalence.

% \vspace{-1mm}
\subsection{Locating Latent Variables via HRC Clustering}
% \vspace{-1mm}
\begin{wrapfigure}{r}{0.6\textwidth}
\centering
% \vspace{-8pt}
\scalebox{0.58}{
\begin{tikzpicture}[
    >={Stealth[length=4mm,width=2.8mm]},
    line width=1pt,
    latent/.style={
        draw,
        circle,
        minimum size=12mm,
        inner sep=0pt,
        font=\large\bfseries
    },
    obs/.style={
        draw,
        circle,
        minimum size=11mm,
        inner sep=0pt,
        fill=gray!20,
        font=\large\bfseries
    },
    groupbox/.style={
        draw,
        rounded corners,
        dashed,
        inner sep=6pt
    }
]

% latent variables
\node[latent] (L1) at (0,0) {$L_1$};
\node[latent] (L2) at (4.8,0) {$L_2$};
\node[latent] (L3) at (9.6,0) {$L_3$};

% latent causal relations
\draw[->] (L1) -- (L2);
\draw[->] (L2) -- (L3);

% observed variables for L1
\node[obs] (X1) at (-1.8,-2.3) {$X_1$};
\node[obs] (X2) at (-0.6,-2.3) {$X_2$};
\node[obs] (X3) at ( 0.6,-2.3) {$X_3$};
\node[obs] (X4) at ( 1.8,-2.3) {$X_4$};

% observed variables for L2
\node[obs] (X5) at (3.6,-2.3) {$X_5$};
\node[obs] (X6) at (4.8,-2.3) {$X_6$};
\node[obs] (X7) at (6.0,-2.3) {$X_7$};

% observed variables for L3
\node[obs] (X8)  at (8.4,-2.3) {$X_8$};
\node[obs] (X9)  at (9.6,-2.3) {$X_9$};
\node[obs] (X10) at (10.8,-2.3) {$X_{10}$};

% measurement edges
\draw[->] (L1) -- (X1);
\draw[->] (L1) -- (X2);
\draw[->] (L1) -- (X3);
\draw[->] (L1) -- (X4);

\draw[->] (L2) -- (X5);
\draw[->] (L2) -- (X6);
\draw[->] (L2) -- (X7);

\draw[->] (L3) -- (X8);
\draw[->] (L3) -- (X9);
\draw[->] (L3) -- (X10);

% cluster boxes
\node[groupbox, fit=(X1)(X2)(X3)(X4)] (C1) {};
\node[groupbox, fit=(X5)(X6)(X7)] (C2) {};
\node[groupbox, fit=(X8)(X9)(X10)] (C3) {};

% cluster labels
\node[font=\Large] at (C1.south) [below=4pt] {$C_1$: pure cluster of $L_1$};
\node[font=\Large] at (C2.south) [below=4pt] {$C_2$: pure cluster of $L_2$};
\node[font=\Large] at (C3.south) [below=4pt] {$C_3$: pure cluster of $L_3$};

\end{tikzpicture}
}
% \vspace{-5pt}
\caption{Illustration of HRC-based latent location and latent causal discovery. The observed variables form three pure measurement clusters $C_1,C_2,C_3$, corresponding to the key latents $L_1,L_2,L_3$, respectively. The latent variables may be causally related, e.g., $L_1\to L_2\to L_3$.}
\label{fig:hrc_location_example}
% \vspace{-5pt}
\end{wrapfigure}
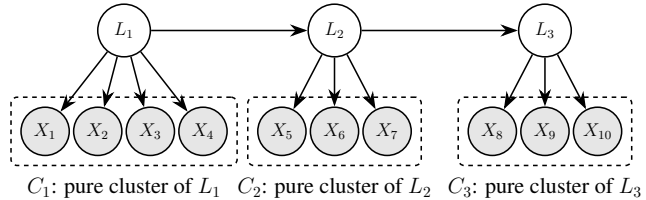
The first task is latent-variable location, namely identifying which observed variables are associated with the same key latent. We follow the Find One Factor Clusters (FOFC in short) framework~\citep{kummerfeld2016causal} and adapt it to the nonlinear setting using the HRC constraint. At a high level, the algorithm searches over triples and quartets of observed variables. A candidate triple is retained as pure only when adding another measured variable produces the rank pattern expected from a one-factor measurement structure; otherwise, the triple is rejected as mixed. In our version, this rank pattern is tested by HRC: a quartet is accepted as a nonlinear one-factor quartet if all three pairwise partitions have cross-Hessian rank one. Retained triples are then merged into maximal clusters, each of which is interpreted as the location of one latent variable. Please refer to Appendix~\ref{app:fofc_details} for more details of the FOFC algorithm.

To make this procedure concrete, we further provide an example in Figure~\ref{fig:hrc_location_example}. The observed variables form three pure clusters: $C_1=\{X_1,X_2,X_3,X_4\}$ for $L_1$, $C_2=\{X_5,X_6,X_7\}$ for $L_2$, and $C_3=\{X_8,X_9,X_{10}\}$ for $L_3$, while the latent variables themselves may be causally related, e.g., $L_1\to L_2\to L_3$. For example, the triple $\{X_1,X_2,X_3\}$ is retained because adding $X_4$ forms a rank-one quartet explained by the single key latent $L_1$; similarly, triples inside $C_2$ and $C_3$ are retained and then merged into their corresponding maximal clusters.
The above clustering procedure is motivated by two theoretical ingredients. First, Theorem~\ref{thm:cross-hessian-constraints} implies a rank-one cross-Hessian for pairwise splits whose minimum latent separator is the same one-dimensional key latent, under its stated assumptions. This provides the oracle rank-one criterion used by the nonlinear FOFC analogue, but does not by itself prove end-to-end correctness of the clustering procedure. Second, Proposition~\ref{pro:unique_latent} establishes uniqueness of the latent explanation conditional on a recovered pure cluster. Accordingly, the final HRC-PC theorem states the downstream HRC-PC guarantee conditional on correct recovery of the pure measurement model.
\begin{proposition}[\textbf{Uniqueness of the Key Latent} \citep{kummerfeld2016causal}]
\label{pro:unique_latent}Let $C\subseteq\mathcal X$ be a set of observed variables with $|C|\geq 2$. Suppose that $L_1$ is a key latent for $C$. Under the faithfulness and non-degeneracy conditions of the measurement model, there does not exist another latent variable $L_2\neq L_1$, with $L_2$ not perfectly correlated with $L_1$, such that $L_2$ is also a key latent for the same cluster $C$.
% \vspace{-1mm}
\end{proposition}
This proposition ensures that the latent explanation of a recovered pure cluster is unique up to perfect correlation. Therefore, once HRC clustering identifies a pure one-factor cluster, the cluster can be interpreted as the location of a distinct latent variable, rather than as an ambiguous group that could be assigned to multiple unrelated latents.

% This lemma guarantees that once HRC clustering identifies a pure one-factor cluster, the latent variable associated with this cluster is unique up to perfect correlation. Hence, the output clusters can be interpreted as the locations of distinct latent variables, rather than ambiguous groups that could be assigned to multiple unrelated latents. Combining this uniqueness property with the rank-one characterization from Theorem~2, the first stage produces a pure measurement model that can be used as input to the latent causal discovery step.

\subsection{Learning Causal Structure among Latent Variables}

After the first stage, each recovered cluster $C_k$ provides measured indicators associated with a key latent $L_k$. The second task is to learn the causal structure among these latent variables. We use pure indicators as observed surrogates for the corresponding latent variables. To test latent conditional independence, we construct two observed groups from these indicators and check whether their cross-Hessian rank equals the dimension of the conditioning latent variables.
% This idea follows rank-based latent causal discovery \citep{silva2006learning}, but uses cross-Hessian rank rather than covariance rank.
% 5.5Following the idea of MIMBuild and its PC-style variants \citep{silva2006learning}, 这句话拿掉，然后说我们的basic idea是xxxx
%  In our nonlinear setting, the classical
% covariance-rank test is replaced by the HRC rank test.这句话去掉，包括后面

In general, to test whether two latent variables $L_i$ and $L_j$ are conditionally independent given a latent set $Q\subseteq\mathcal L\setminus\{L_i,L_j\}$ with $|Q|=q$, we construct two disjoint observed groups $A$ and $B$ from their measured indicators. Specifically, $A$ contains one indicator of $L_i$, $B$ contains one indicator of $L_j$, and both $A$ and $B$ contain one indicator for each latent variable in $Q$. Thus, each conditioning latent is represented on both sides of the partition, while the two endpoint latents are represented on different sides. Keeping the original rank-based procedure, we simply replace the covariance-rank condition with the Hessian-rank condition by checking whether the cross-Hessian $\mathbf H_{A,B}$ has rank $q$.

For example, in Figure~\ref{fig:hrc_location_example}, the recovered clusters $C_1,C_2,C_3$ locate the key latents $L_1,L_2,L_3$, respectively. To test whether $L_1 \perp\!\!\!\perp L_3 \mid L_2$, we choose measured variables from the corresponding clusters and split them into two observed groups. Specifically, we may set $A=\{X_1,X_5\}$ and $B=\{X_8,X_6\}$, where $X_1$ and $X_8$ are indicators of the endpoint latents $L_1$ and $L_3$, while $X_5$ and $X_6$ are two indicators of the conditioning latent $L_2$ placed on different sides of the partition. We then determine whether the rank of $\mathbf H_{A,B}$ equals one, corresponding to the one-dimensional conditioning set $\{L_2\}$. The validity of this surrogate rank test is formalized as follows.

% \begin{theorem}[Latent Conditional Independence via HRC]
% Assume a correctly identified pure nonlinear one-factor measurement model satisfying the assumptions of Theorem~2. Let $L_i$ and $L_j$ be two distinct latent variables, and let $Q\subseteq L\setminus\{L_i,L_j\}$ be a set of $q$ latent variables. Suppose that each latent variable in $Q$ has at least two pure indicators. Let $A$ and $B$ be disjoint observed groups constructed as above. Then
% \begin{equation}
% \small
%     L_i \perp\!\!\!\perp L_j \mid Q
% \quad\Longleftrightarrow\quad
% \mathrm{rank}(H_{A,B})=q .
% \end{equation}
% \end{theorem}
\begin{theorem}[\textbf{Latent Conditional Independence via HRC}]
\label{thm:latent-ci-hrc}
Assume a correctly identified pure nonlinear one-factor measurement model whose latent distribution is Markov and faithful to its latent DAG. Let $L_i$ and $L_j$ be distinct latent variables, and let
$Q\subseteq\mathcal L\setminus\{L_i,L_j\}$ with $|Q|=q$. Suppose every latent variable in $Q$ has at least two pure indicators. Construct disjoint observed groups $A$ and $B$ so that $A$ contains one pure indicator of $L_i$, $B$ contains one pure indicator of $L_j$, and each latent in $Q$ is represented by one pure indicator in each group. Suppose Assumptions~\ref{assump:regularity}--\ref{assump:nondegeneracy} hold for these groups with respect to their minimum latent separator. Then
\begin{equation}
L_i\perp\!\!\!\perp L_j\mid Q
\quad\Longleftrightarrow\quad
\operatorname{rank}\!\left(\mathbf H_{A,B}\right)=q.
\end{equation}
\end{theorem}
% 5.5Let $A$ and $B$ be disjoint observed groups constructed as above.有点模糊，A，B没有提前定义
Proof can be found in Appendix \ref{app:the5}. This theorem provides a conditional-independence test for latent variables using only observed surrogate groups. Instead of directly testing independence among unobserved variables, we examine the rank of the cross-Hessian block between the corresponding measured indicators. When this rank equals $|Q|$, the dependence between the two observed groups is mediated only by the conditioning latent set $Q$, which corresponds to $L_i \perp\!\!\!\perp L_j\mid Q$. We then incorporate this test into a PC-style constraint-based search over the recovered latent variables, which we call HRC-PC. Its consistency guarantee is stated below.

\begin{theorem}[\textbf{Correctness of HRC-PC}]
\label{thm:hrc-pc-correctness}
\textcolor{black}{Suppose Assumptions~\ref{assump:regularity}--\ref{assump:nondegeneracy} hold for every HRC test performed by HRC-PC, and the latent DAG satisfies the causal Markov and ordinary causal faithfulness assumptions. If HRC-based FOFC correctly recovers the pure measurement clusters and their key latents with the indicators required by HRC-PC, and the empirical HRC rank tests are consistent, then HRC-PC asymptotically recovers the Markov equivalence class of the latent DAG.}
\end{theorem}

Please see Appendix~\ref{app:the6} for the proof. \textcolor{black}{The result follows by combining correct recovery of the pure measurement clusters by HRC-based FOFC with consistent HRC rank tests. Once the key latents are located, HRC-PC applies standard constraint-based search using the HRC tests as conditional-independence decisions. Under the stated assumptions, it therefore asymptotically recovers the Markov equivalence class of the latent DAG. }
% \begin{algorithm}[htbp]
% \caption{Cross-Hessian PC-MIMBUILD (CH-PC-MIMBUILD)}
% \label{alg:ch_pc_mimbuild_main}
% \textbf{Input:} Pure clusters $\mathbf{C} = \{C_1, \dots, C_k\}$ representing latents $\mathbf{L} = \{L_1, \dots, L_k\}$ ($|C_i| \ge 2$), Rank tolerance $\epsilon$.\\
% \textbf{Output:} CPDAG $G_{struct}$ representing the structural model over $\mathbf{L}$.
% \begin{algorithmic}[1]
% \STATE $G, SepSet \gets \textsc{SkeletonDiscovery}(\mathbf{L}, \mathbf{C}, \epsilon)$
% \STATE $G_{struct} \gets \textsc{EdgeOrientation}(G, SepSet)$
% \RETURN $G_{struct}$
% \end{algorithmic}

% \vspace{0.2cm}
% \textit{*Note on Procedures:}\\
% \textit{\textsc{SkeletonDiscovery}: Iteratively tests latent conditional independencies using the Cross-Hessian rank test to find the undirected skeleton and separation sets.}\\
% \textit{\textsc{EdgeOrientation}: Applies unshielded collider identification and Meek's rules to establish edge directions.}
% \end{algorithm}
\begin{table}[t]
\centering
\caption{Results under different functional forms.}
\vspace{-2mm}
\renewcommand{\arraystretch}{1.15}
\setlength{\tabcolsep}{5pt}
\resizebox{\linewidth}{!}{%
\begin{tabular}{lcccccccccc}
\toprule
& \multicolumn{2}{c}{HRC} & \multicolumn{2}{c}{Prashant et al. [2025]} & \multicolumn{2}{c}{Kong et al. [2023] } & \multicolumn{2}{c}{Silva et al. [2006]} & \multicolumn{2}{c}{Kummerfeld and Ramsey [2016]} \\
\cmidrule(lr){2-3} \cmidrule(lr){4-5} \cmidrule(lr){6-7} \cmidrule(lr){8-9} \cmidrule(lr){10-11}
& F1$\uparrow$ & SHD$\downarrow$ & F1$\uparrow$ & SHD$\downarrow$ & F1$\uparrow$ & SHD$\downarrow$ & F1$\uparrow$ & SHD$\downarrow$ & F1$\uparrow$ & SHD$\downarrow$ \\
\midrule
Mixed
& \textbf{0.985} {\scriptsize$\pm$ 0.052} & \textbf{0.093} {\scriptsize$\pm$ 0.314}
& 0.437 {\scriptsize$\pm$ 0.326} & 2.290 {\scriptsize$\pm$ 0.907}
& 0.274 {\scriptsize$\pm$ 0.269} & 2.953 {\scriptsize$\pm$ 1.011}
& 0.224 {\scriptsize$\pm$ 0.284} & 3.363 {\scriptsize$\pm$ 1.276}
& 0.120 {\scriptsize$\pm$ 0.244} & 3.060 {\scriptsize$\pm$ 0.852} \\
Sin
& \textbf{0.966} {\scriptsize$\pm$ 0.088} & \textbf{0.193} {\scriptsize$\pm$ 0.507}
& 0.410 {\scriptsize$\pm$ 0.318} & 2.387 {\scriptsize$\pm$ 0.856}
& 0.304 {\scriptsize$\pm$ 0.248} & 2.827 {\scriptsize$\pm$ 0.941}
& 0.102 {\scriptsize$\pm$ 0.196} & 3.593 {\scriptsize$\pm$ 1.012}
& 0.068 {\scriptsize$\pm$ 0.196} & 3.150 {\scriptsize$\pm$ 0.810} \\
Sinh
& \textbf{0.996} {\scriptsize$\pm$ 0.026} & \textbf{0.023} {\scriptsize$\pm$ 0.151}
& 0.448 {\scriptsize$\pm$ 0.314} & 2.293 {\scriptsize$\pm$ 0.866}
& 0.237 {\scriptsize$\pm$ 0.263} & 3.163 {\scriptsize$\pm$ 0.980}
& 0.112 {\scriptsize$\pm$ 0.217} & 3.626 {\scriptsize$\pm$ 1.028}
& 0.110 {\scriptsize$\pm$ 0.235} & 3.107 {\scriptsize$\pm$ 0.934} \\
Softplus
& \textbf{0.974} {\scriptsize$\pm$ 0.083} & \textbf{0.167} {\scriptsize$\pm$ 0.469}
& 0.466 {\scriptsize$\pm$ 0.321} & 2.217 {\scriptsize$\pm$ 0.905}
& 0.194 {\scriptsize$\pm$ 0.265} & 3.293 {\scriptsize$\pm$ 0.954}
& 0.422 {\scriptsize$\pm$ 0.337} & 2.826 {\scriptsize$\pm$ 1.511}
& 0.101 {\scriptsize$\pm$ 0.231} & 3.153 {\scriptsize$\pm$ 0.883} \\
Tanh
& \textbf{0.950} {\scriptsize$\pm$ 0.105} & \textbf{0.303} {\scriptsize$\pm$ 0.632}
& 0.449 {\scriptsize$\pm$ 0.324} & 2.267 {\scriptsize$\pm$ 0.912}
& 0.307 {\scriptsize$\pm$ 0.252} & 2.803 {\scriptsize$\pm$ 0.956}
& 0.206 {\scriptsize$\pm$ 0.308} & 3.526 {\scriptsize$\pm$ 1.488}
& 0.107 {\scriptsize$\pm$ 0.234} & 3.063 {\scriptsize$\pm$ 0.899} \\
\midrule
Average
& \textbf{0.974} {\scriptsize$\pm$ 0.071} & \textbf{0.156} {\scriptsize$\pm$ 0.414}
& 0.442 {\scriptsize$\pm$ 0.320} & 2.291 {\scriptsize$\pm$ 0.889}
& 0.263 {\scriptsize$\pm$ 0.259} & 3.008 {\scriptsize$\pm$ 0.968}
& 0.213 {\scriptsize$\pm$ 0.297} & 3.387 {\scriptsize$\pm$ 1.313}
& 0.101 {\scriptsize$\pm$ 0.228} & 3.107 {\scriptsize$\pm$ 0.876} \\
\bottomrule
\end{tabular}%
}
\end{table}

\begin{table}[t]
\centering
% \vspace{-2mm}
\caption{Results under different latent dimensions.}
\vspace{-2mm}
\renewcommand{\arraystretch}{1.15}
\setlength{\tabcolsep}{5pt}
\resizebox{\linewidth}{!}{%
\begin{tabular}{c|cccccccccc}
\toprule
& \multicolumn{2}{c}{HRC} & \multicolumn{2}{c}{Prashant et al. [2025]} & \multicolumn{2}{c}{Kong et al. [2023]} & \multicolumn{2}{c}{Silva et al. [2006]} & \multicolumn{2}{c}{Kummerfeld and Ramsey [2016]} \\
\cmidrule(lr){2-3} \cmidrule(lr){4-5} \cmidrule(lr){6-7} \cmidrule(lr){8-9} \cmidrule(lr){10-11}
\makecell[c]{Dimension}
& F1$\uparrow$ & SHD$\downarrow$
& F1$\uparrow$ & SHD$\downarrow$
& F1$\uparrow$ & SHD$\downarrow$
& F1$\uparrow$ & SHD$\downarrow$
& F1$\uparrow$ & SHD$\downarrow$ \\
\midrule
% 3
% & \textbf{0.950} {\scriptsize$\pm$ 0.105} & \textbf{0.303} {\scriptsize$\pm$ 0.632}
% & 0.449 {\scriptsize$\pm$ 0.324} & 2.267 {\scriptsize$\pm$ 0.912}
% & 0.307 {\scriptsize$\pm$ 0.252} & 2.803 {\scriptsize$\pm$ 0.956}
% & 0.216 {\scriptsize$\pm$ 0.263} & 3.163 {\scriptsize$\pm$ 0.990}
% & 0.107 {\scriptsize$\pm$ 0.234} & 3.063 {\scriptsize$\pm$ 0.899} \\
4
& \textbf{0.962} {\scriptsize$\pm$ 0.099} & \textbf{0.203} {\scriptsize$\pm$ 0.493}
& 0.487 {\scriptsize$\pm$ 0.259} & 2.844 {\scriptsize$\pm$ 1.505}
& 0.502 {\scriptsize$\pm$ 0.238} & 2.404 {\scriptsize$\pm$ 1.053}
& 0.214 {\scriptsize$\pm$ 0.297} & 3.387 {\scriptsize$\pm$ 1.313}
& 0.254 {\scriptsize$\pm$ 0.362} & 3.500 {\scriptsize$\pm$ 1.743} \\
6
& \textbf{0.952} {\scriptsize$\pm$ 0.085} & \textbf{0.469} {\scriptsize$\pm$ 0.790}
& 0.417 {\scriptsize$\pm$ 0.176} & 5.333 {\scriptsize$\pm$ 1.155}
& 0.513 {\scriptsize$\pm$ 0.159} & 4.716 {\scriptsize$\pm$ 0.958}
& 0.216 {\scriptsize$\pm$ 0.220} & 5.546 {\scriptsize$\pm$ 1.648}
& 0.210 {\scriptsize$\pm$ 0.091} & 5.328 {\scriptsize$\pm$ 0.165} \\
8
& \textbf{0.937} {\scriptsize$\pm$ 0.100} & \textbf{0.972} {\scriptsize$\pm$ 1.021}
& 0.264 {\scriptsize$\pm$ 0.161} & 8.688 {\scriptsize$\pm$ 1.899}
& 0.184 {\scriptsize$\pm$ 0.130} & 8.732 {\scriptsize$\pm$ 1.244}
& 0.204 {\scriptsize$\pm$ 0.186} & 8.807 {\scriptsize$\pm$ 2.853}
& 0.191 {\scriptsize$\pm$ 0.061} & 8.452 {\scriptsize$\pm$ 0.337} \\
10
& \textbf{0.923} {\scriptsize$\pm$ 0.110} & \textbf{1.214} {\scriptsize$\pm$ 1.542}
& 0.154 {\scriptsize$\pm$ 0.105} & 12.538 {\scriptsize$\pm$ 1.866}
& 0.073 {\scriptsize$\pm$ 0.098} & 11.084 {\scriptsize$\pm$ 1.307}
& 0.163 {\scriptsize$\pm$ 0.138} & 11.276 {\scriptsize$\pm$ 3.176}
& 0.127 {\scriptsize$\pm$ 0.048} & 11.036 {\scriptsize$\pm$ 0.433} \\
\bottomrule
\end{tabular}%
}
\end{table}

\section{Experiments}
% \vspace{-1mm}
\subsection{Synthetic Experiment}
% \vspace{-1mm}
\textbf{Setup.} We evaluate HRC on synthetic nonlinear latent variable models with pure measurement variables. Each observed variable is generated from its latent parent through a nonlinear function. We consider five different nonlinear functions: $\sin$, $\sinh$, $\tanh$, softplus, and a mixed setting that assigns them independently across observed variables. We vary the latent dimensionality over $\{4,6,8,10\}$. For each dimension, we use 10 random latent graphs and 10 seeds per graph. We consider sample sizes $N\in\{200,1000,2000,\ldots,9000\}$. For latent variable location, we use F1 for evaluation. For latent causal discovery, we evaluate recovery of the latent Markov equivalence class using F1 and structural Hamming distance (SHD). We report the mean and standard deviation over all graph structures and random-seed combinations. Please refer to Appendix \ref{app:synthetic_data} for details.

\textbf{Baselines.} For the latent-variable location task, all methods use the FOFC framework; our method replaces the original covariance-rank constraint with the proposed Hessian Rank Constraint, while the baseline keeps the classical covariance-rank test. For latent causal discovery, we compare with \cite{prashant2025differentiable,kong2023identification,kummerfeld2016causal} and \cite{silva2006learning}. These methods represent recent nonlinear latent causal discovery approaches as well as classical covariance-rank-based methods for linear latent variable models and one-factor measurement models.

\textbf{Results.}
Figure 4 reports the clustering results for latent-variable location. Since both methods use the same FOFC framework, the comparison reflects the effect of the rank constraint. HRC achieves higher clustering F1 than the baseline, suggesting that HRC provides a better criterion for identifying pure measurement clusters under nonlinear mixing procedures.
\begin{wrapfigure}{r}{0.44\textwidth}
\vspace{-0.4\baselineskip}
    \centering
    \includegraphics[width=0.98\linewidth]{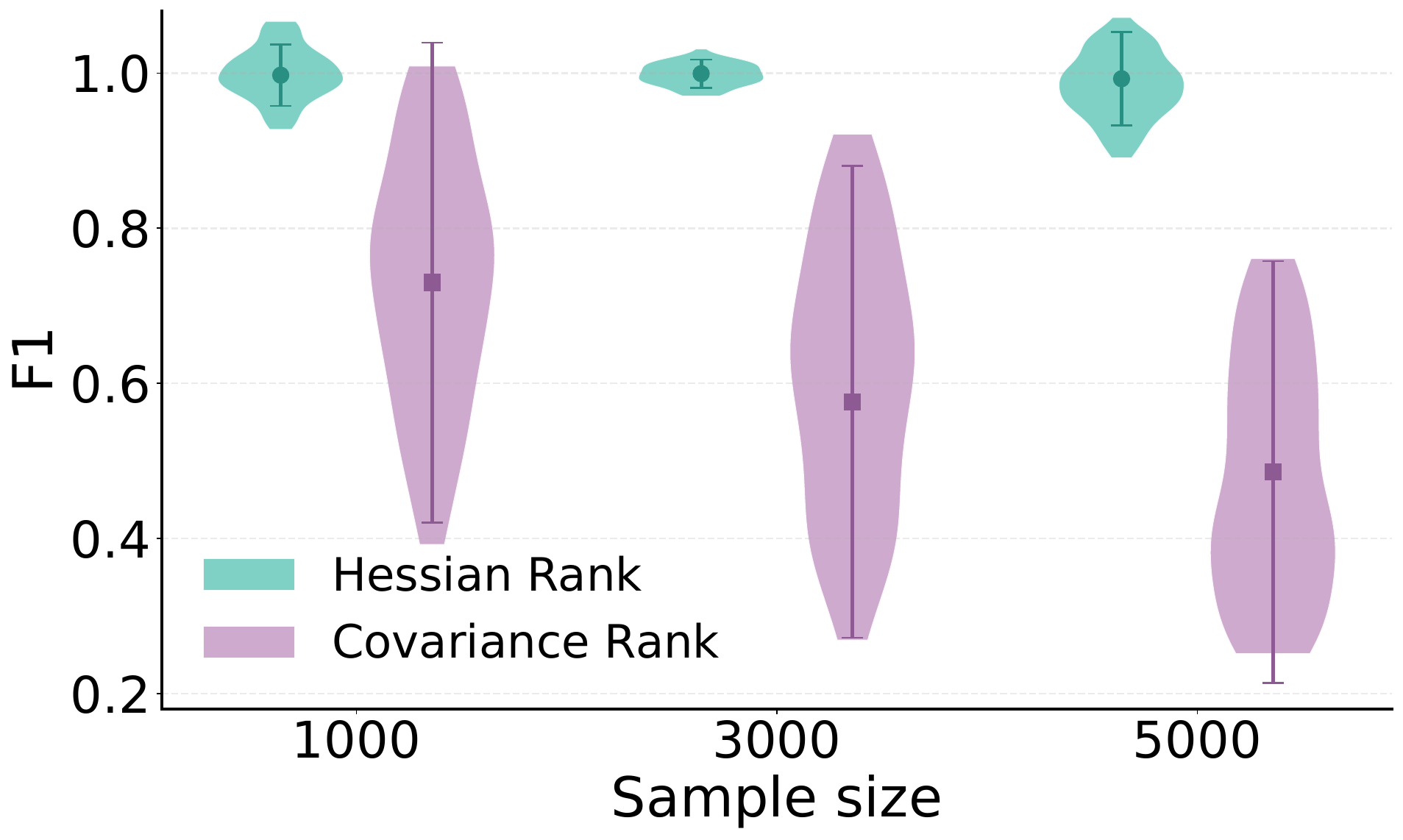}
    \caption{\small Clustering results between Hessian Rank constraint and covariance rank constraint.}
    \label{fig:clustering}
\vspace{-0.6\baselineskip}
\end{wrapfigure}
Tables 1 and 2 report the results for causal discovery among latent variables. HRC consistently outperforms all baselines across different nonlinear mechanisms and latent dimensions. In Table 1, HRC achieves the best results under every nonlinear measurement function, with an average F1 of 0.974 and an average SHD of 0.156, while the strongest baseline only reaches 0.442 F1 on average and has much larger SHD. This advantage remains stable across different nonlinearities, indicating that HRC can capture latent structural information beyond linear covariance dependence. Figure 4 further evaluates the effect of sample size. As the number of samples increases, HRC maintains high precision, recall, and F1. In contrast, the baselines remain substantially worse across the whole sample-size range, showing that HRC benefits more effectively from additional observations in nonlinear latent causal discovery. Please refer to Appendix \ref{app:more_exp} for experiments regarding varying noise levels and sampling regions of latent variables.

% Tables 1 and 2 show that HRC consistently outperforms all baselines across both nonlinear measurement functions and latent dimensions. In Table 1, HRC achieves the best results under every nonlinear mechanism, with an average F1 of 0.974 and an average SHD of 0.156, while the best-performing baseline on average reaches only 0.442 F1 with substantially larger SHD. This advantage remains stable across all tested nonlinearities, including the tanh setting, where HRC still achieves 0.950 F1 and 0.303 SHD. These results suggest that HRC captures a structural signal that is robust to heterogeneous nonlinear measurement functions.

% Table 2 further shows that HRC scales favorably with latent dimensionality. From 3 to 10 latent variables, HRC maintains high F1 scores and consistently achieves the lowest SHD among all methods. Although SHD increases as the latent graph becomes larger and more complex, the degradation of HRC is much milder than that of the baselines, whose F1 scores drop in higher-dimensional settings.

% Figure 2 shows that HRC also improves with sample size. As the number of samples increases from 1,000 to 9,000, HRC maintains high precision, recall, and F1, while its SHD steadily decreases. In contrast, the baselines remain substantially worse throughout the entire sample-size range. 
% This is consistent with the intuition behind HRC: with more data, the empirical cross-Hessian estimates become more stable, leading to more reliable rank estimation and more accurate recovery of the latent Markov equivalence class.
\begin{wrapfigure}{r}{0.44\textwidth}
\vspace{-0.4\baselineskip}
    \centering
    \includegraphics[width=0.98\linewidth]{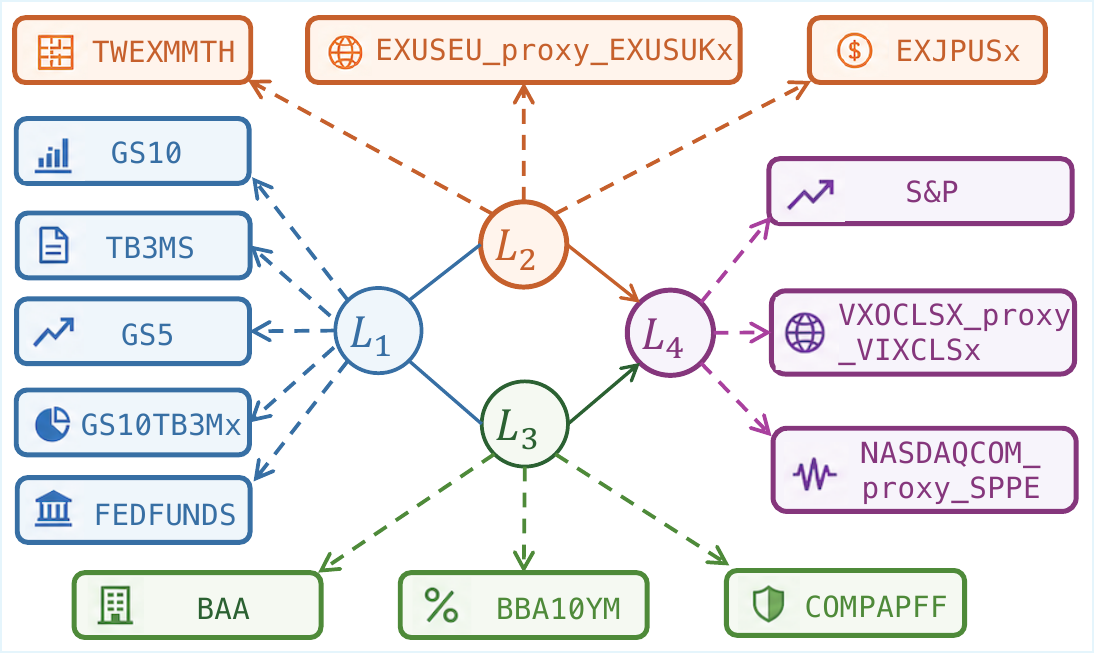}
    \caption{\small Recovered latent structure on the FRED-MD macroeconomic dataset.}
    \label{fig:real_world}
\vspace{-0.6\baselineskip}
\end{wrapfigure}
We further consider the real-world dataset, namely the FRED-MD macroeconomic database from the Federal Reserve Bank of St. Louis.\footnote{\url{https://www.stlouisfed.org/research/economists/mccracken/fred-databases}} We select 14 macroeconomic variables and group them into four latent factors according to economic domain knowledge. Specifically, FEDFUNDS, TB3MS, GS10, GS5, and GS10TB3Mx are used as indicators of an interest-rate factor; BAA, BAA10YM, and COMPAPFF are used as indicators of a credit factor; TWEXMMTH, EXUSEU, and EXJPUSx are used as indicators of an exchange-rate factor; and S\&P 500, NASDAQCOM, and VXOCLSX are used as indicators of an equity/volatility factor. More information can be found in Appendix \ref{app:real_data}.

Figure~\ref{fig:real_world} shows that HRC recovers an economically interpretable latent structure. The recovered factors correspond well to the predefined economic groups, including interest rates, credit conditions, exchange rates, and equity/volatility. Moreover, the learned latent graph connects the interest-rate factor with the credit and exchange-rate factors, which is consistent with the role of monetary and yield-curve conditions in shaping credit-market and currency movements. The credit and exchange-rate factors are further connected to the equity/volatility factor, suggesting that stock-market fluctuations are jointly associated with credit stress and foreign-exchange movements. These results indicate that HRC can recover meaningful latent structures from real-world data.

% \vspace{-1mm}
\section{Conclusion}
% \vspace{-1mm}
We proposed the Hessian Rank Constraint (HRC), a rank-based criterion for nonlinear latent causal discovery. HRC uses the cross-Hessian rank of the observed log-density to characterize latent separation, extending classical covariance-rank ideas beyond linear measurement models. Based on this criterion, we incorporated HRC into both latent-variable location and latent causal discovery: it can be used in the FOFC framework to recover pure measurement clusters, and further used in a PC-style procedure to learn causal relations among latent variables up to Markov equivalence. Experiments on synthetic data show that HRC consistently outperforms covariance-based methods and recent nonlinear baselines under different nonlinear functions, latent dimensions, and sample sizes. Results on the FRED-MD dataset further suggest that HRC can recover interpretable latent structures from real-world macroeconomic data.

\let\arxivoriginaltextcolor\textcolor
\renewcommand{\textcolor}[2]{%
  \ifstrequal{#1}{blue}{\arxivoriginaltextcolor{black}{#2}}{\arxivoriginaltextcolor{#1}{#2}}%
}
\hypersetup{linkcolor=black,citecolor=black,urlcolor=black}
\bibliographystyle{plainnat}
\bibliography{main}

@inproceedings{kummerfeld2016causal,
  title={Causal clustering for 1-factor measurement models},
  author={Kummerfeld, Erich and Ramsey, Joseph},
  booktitle={Proceedings of the 22nd ACM SIGKDD international conference on knowledge discovery and data mining},
  pages={1655--1664},
  year={2016}
}

@book{spirtes2000causation,
  title={Causation, prediction, and search},
  author={Spirtes, Peter and Glymour, Clark N and Scheines, Richard},
  year={2000},
  publisher={MIT press}
}

@article{zhang2008completeness,
  title={On the completeness of orientation rules for causal discovery in the presence of latent confounders and selection bias},
  author={Zhang, Jiji},
  journal={Artificial Intelligence},
  volume={172},
  number={16-17},
  pages={1873--1896},
  year={2008},
  publisher={Elsevier}
}

@article{colombo2012learning,
  title={Learning high-dimensional directed acyclic graphs with latent and selection variables},
  author={Colombo, Diego and Maathuis, Marloes H and Kalisch, Markus and Richardson, Thomas S},
  journal={The Annals of Statistics},
  pages={294--321},
  year={2012},
  publisher={JSTOR}
}

@article{akbari2021recursive,
  title={Recursive causal structure learning in the presence of latent variables and selection bias},
  author={Akbari, Sina and Mokhtarian, Ehsan and Ghassami, AmirEmad and Kiyavash, Negar},
  journal={Advances in Neural Information Processing Systems},
  volume={34},
  pages={10119--10130},
  year={2021}
}

@article{huang2022latent,
  title={Latent hierarchical causal structure discovery with rank constraints},
  author={Huang, Biwei and Low, Charles Jia Han and Xie, Feng and Glymour, Clark and Zhang, Kun},
  journal={Advances in neural information processing systems},
  volume={35},
  pages={5549--5561},
  year={2022}
}

@article{silva2006learning,
	title={Learning the structure of linear latent variable models},
	author={Silva, Ricardo and Scheine, Richard and Glymour, Clark and Spirtes, Peter},
	journal={Journal of Machine Learning Research},
	volume={7},
	number={Feb},
	pages={191--246},
	year={2006}
}

@inproceedings{xie2022identification,
  title={Identification of linear {{\textcolor{blue}{non-Gaussian}}} latent hierarchical structure},
  author={Xie, Feng and Huang, Biwei and Chen, Zhengming and He, Yangbo and Geng, Zhi and Zhang, Kun},
  booktitle={International Conference on Machine Learning},
  pages={24370--24387},
  year={2022},
  organization={PMLR}
}

@article{dong2023versatile,
  title={A versatile causal discovery framework to allow causally-related hidden variables},
  author={Dong, Xinshuai and Huang, Biwei and Ng, Ignavier and Song, Xiangchen and Zheng, Yujia and Jin, Songyao and Legaspi, Roberto and Spirtes, Peter and Zhang, Kun},
  journal={arXiv preprint arXiv:2312.11001},
  year={2023}
}

@article{shimizu2009estimation,
  title={Estimation of linear {{\textcolor{blue}{non-Gaussian}}} acyclic models for latent factors},
  author={Shimizu, Shohei and Hoyer, Patrik O and Hyv{\"a}rinen, Aapo},
  journal={Neurocomputing},
  volume={72},
  number={7-9},
  pages={2024--2027},
  year={2009},
  publisher={Elsevier}
}

@article{xie2020generalized,
  title={Generalized independent noise condition for estimating latent variable causal graphs},
  author={Xie, Feng and Cai, Ruichu and Huang, Biwei and Glymour, Clark and Hao, Zhifeng and Zhang, Kun},
  journal={Advances in neural information processing systems},
  volume={33},
  pages={14891--14902},
  year={2020}
}

@article{cai2019triad,
  title={Triad constraints for learning causal structure of latent variables},
  author={Cai, Ruichu and Xie, Feng and Glymour, Clark and Hao, Zhifeng and Zhang, Kun},
  journal={Advances in neural information processing systems},
  volume={32},
  year={2019}
}

@article{adams2021identification,
  title={Identification of partially observed linear causal models: Graphical conditions for the {{\textcolor{blue}{non-Gaussian}}} and heterogeneous cases},
  author={Adams, Jeffrey and Hansen, Niels and Zhang, Kun},
  journal={Advances in Neural Information Processing Systems},
  volume={34},
  pages={22822--22833},
  year={2021}
}

@inproceedings{chen2022identification,
  title={Identification of linear latent variable model with arbitrary distribution},
  author={Chen, Zhengming and Xie, Feng and Qiao, Jie and Hao, Zhifeng and Zhang, Kun and Cai, Ruichu},
  booktitle={Proceedings of the AAAI Conference on Artificial Intelligence},
  volume={36},
  number={6},
  pages={6350--6357},
  year={2022}
}

@inproceedings{anandkumar2013learning,
  title={Learning linear {{\textcolor{blue}{Bayesian}}} networks with latent variables},
  author={Anandkumar, Animashree and Hsu, Daniel and Javanmard, Adel and Kakade, Sham},
  booktitle={International Conference on Machine Learning},
  pages={249--257},
  year={2013},
  organization={PMLR}
}

@inproceedings{chen2025identification,
  title={Identification of Latent Confounders via Investigating the Tensor Ranks of the Nonlinear Observations},
  author={Chen, Zhengming and Xia, Yewei and Xie, Feng and Qiao, Jie and Hao, Zhifeng and Cai, Ruichu and Zhang, Kun},
  booktitle={Forty-second International Conference on Machine Learning},
  year={2025}
}

@article{cui2018learning,
  title={Learning the causal structure of copula models with latent variables},
  author={Cui, Ruifei and Groot, Perry and Schauer, Moritz and Heskes, Tom},
  year={2018},
  publisher={Corvallis: AUAI Press}
}

@inproceedings{zeng2021causal,
  title={Causal discovery with multi-domain {{\textcolor{blue}{LiNGAM}}} for latent factors},
  author={Zeng, Yan and Shimizu, Shohei and Cai, Ruichu and Xie, Feng and Yamamoto, Michio and Hao, Zhifeng},
  booktitle={Causal Analysis Workshop Series},
  pages={1--4},
  year={2021},
  organization={PMLR}
}

@article{li2023causal,
  title={Causal discovery from observational and interventional data across multiple environments},
  author={Li, Adam and Jaber, Amin and Bareinboim, Elias},
  journal={Advances in Neural Information Processing Systems},
  volume={36},
  pages={16942--16956},
  year={2023}
}

@article{sturma2023unpaired,
  title={Unpaired multi-domain causal representation learning},
  author={Sturma, Nils and Squires, Chandler and Drton, Mathias and Uhler, Caroline},
  journal={Advances in Neural Information Processing Systems},
  volume={36},
  pages={34465--34492},
  year={2023}
}

@article{zhang2024causal,
  title={Causal representation learning from multiple distributions: A general setting},
  author={Zhang, Kun and Xie, Shaoan and Ng, Ignavier and Zheng, Yujia},
  journal={arXiv preprint arXiv:2402.05052},
  year={2024}
}

@article{kivva2021learning,
  title={Learning latent causal graphs via mixture oracles},
  author={Kivva, Bohdan and Rajendran, Goutham and Ravikumar, Pradeep and Aragam, Bryon},
  journal={Advances in Neural Information Processing Systems},
  volume={34},
  pages={18087--18101},
  year={2021}
}

@article{kong2023identification,
  title={Identification of nonlinear latent hierarchical models},
  author={Kong, Lingjing and Huang, Biwei and Xie, Feng and Xing, Eric and Chi, Yuejie and Zhang, Kun},
  journal={Advances in Neural Information Processing Systems},
  volume={36},
  pages={2010--2032},
  year={2023}
}

@inproceedings{
prashant2025differentiable,
title={Differentiable Causal Discovery for Latent Hierarchical Causal Models},
author={Parjanya Prajakta Prashant and Ignavier Ng and Kun Zhang and Biwei Huang},
booktitle={The Thirteenth International Conference on Learning Representations},
year={2025},
url={https://openreview.net/forum?id=Bp0HBaMNRl}
}

@article{chen2017tutorial,
  title={A tutorial on kernel density estimation and recent advances},
  author={Chen, Yen-Chi},
  journal={Biostatistics \& Epidemiology},
  volume={1},
  number={1},
  pages={161--187},
  year={2017},
  publisher={Taylor \& Francis}
}

@article{mokhtarian2025recursive,
  title={Recursive causal discovery},
  author={Mokhtarian, Ehsan and Elahi, Sepehr and Akbari, Sina and Kiyavash, Negar},
  journal={Journal of Machine Learning Research},
  volume={26},
  number={61},
  pages={1--65},
  year={2025}
}

@article{kivva2022identifiability, title={Identifiability of deep generative models without auxiliary information}, author={Kivva, Bohdan and Rajendran, Goutham and Ravikumar, Pradeep and Aragam, Bryon}, journal={Advances in Neural Information Processing Systems}, volume={35}, pages={15687--15701}, year={2022} }

@article{spearman1928pearson,
	title={Pearson's contribution to the theory of two factors},
	author={Spearman, Charles},
	journal={British Journal of Psychology. General Section},
	volume={19},
	number={1},
	pages={95--101},
	year={1928},
	publisher={Wiley Online Library}
}

@article{Sullivant-T-separation,
	title={Trek separation for {{\textcolor{blue}{Gaussian}}} graphical models},
	author={Sullivant, Seth and Talaska, Kelli and Draisma, Jan and others},
	journal={The Annals of Statistics},
	volume={38},
	number={3},
	pages={1665--1685},
	year={2010},
	publisher={Institute of Mathematical Statistics}
}

@article{glymour2019review,
  title={Review of causal discovery methods based on graphical models},
  author={Glymour, Clark and Zhang, Kun and Spirtes, Peter},
  journal={Frontiers in genetics},
  volume={10},
  pages={524},
  year={2019},
  publisher={Frontiers Media SA}
}

@article{huang2020causal,
  title={Causal discovery from heterogeneous/nonstationary data},
  author={Huang, Biwei and Zhang, Kun and Zhang, Jiji and Ramsey, Joseph and Sanchez-Romero, Ruben and Glymour, Clark and Sch{\"o}lkopf, Bernhard},
  journal={Journal of Machine Learning Research},
  volume={21},
  number={89},
  pages={1--53},
  year={2020}
}

@article{zhu2019causal,
  title={Causal discovery with reinforcement learning},
  author={Zhu, Shengyu and Ng, Ignavier and Chen, Zhitang},
  journal={arXiv preprint arXiv:1906.04477},
  year={2019}
}

@book{pearl2009causality,
  title={Causality},
  author={Pearl, Judea},
  year={2009},
  publisher={Cambridge {{\textcolor{blue}{University Press}}}}
}

@article{hausman1999independence,
  title={Independence, invariance and the causal {{\textcolor{blue}{Markov}}} condition},
  author={Hausman, Daniel M and Woodward, James},
  journal={The British journal for the philosophy of science},
  volume={50},
  number={4},
  pages={521--583},
  year={1999},
  publisher={Oxford University Press}
}

@article{zhalama2017weakening,
  title={Weakening faithfulness: some heuristic causal discovery algorithms},
  author={Zhalama and Zhang, Jiji and Mayer, Wolfgang},
  journal={International journal of data science and analytics},
  volume={3},
  number={2},
  pages={93--104},
  year={2017},
  publisher={Springer}
}

@article{zhang2012kernel,
  title={Kernel-based conditional independence test and application in causal discovery},
  author={Zhang, Kun and Peters, Jonas and Janzing, Dominik and Sch{\"o}lkopf, Bernhard},
  journal={arXiv preprint arXiv:1202.3775},
  year={2012}
}

@inproceedings{ng2024score,
  title={Score-based causal discovery of latent variable causal models},
  author={Ng, Ignavier and Dong, Xinshuai and Dai, Haoyue and Huang, Biwei and Spirtes, Peter and Zhang, Kun},
  booktitle={Forty-first International Conference on Machine Learning},
  year={2024}
}

@inproceedings{montagna2023scalable,
  title={Scalable causal discovery with score matching},
  author={Montagna, Francesco and Noceti, Nicoletta and Rosasco, Lorenzo and Zhang, Kun and Locatello, Francesco},
  booktitle={Conference on Causal Learning and Reasoning},
  pages={752--771},
  year={2023},
  organization={PMLR}
}

@inproceedings{spirtes2016causal,
  title={Causal discovery and inference: concepts and recent methodological advances},
  author={Spirtes, Peter and Zhang, Kun},
  booktitle={Applied informatics},
  volume={3},
  number={1},
  pages={3},
  year={2016},
  organization={Springer}
}

@inproceedings{huang2018generalized,
  title={Generalized score functions for causal discovery},
  author={Huang, Biwei and Zhang, Kun and Lin, Yizhu and Sch{\"o}lkopf, Bernhard and Glymour, Clark},
  booktitle={Proceedings of the 24th ACM SIGKDD international conference on knowledge discovery \& data mining},
  pages={1551--1560},
  year={2018}
}

@inproceedings{li2025strong,
  title={Strong and Weak Identifiability of Optimization-based Causal Discovery in Non-linear Additive Noise Models},
  author={Li, Mingjia and Qian, Hong and Wang, Tian-Zuo and Zhang, Min and Zhou, Aimin and others},
  booktitle={Forty-second International Conference on Machine Learning},
  year={2025}
}

@article{yin2024optimization,
  title={Optimization-based causal estimation from heterogeneous environments},
  author={Yin, Mingzhang and Wang, Yixin and Blei, David M},
  journal={Journal of Machine Learning Research},
  volume={25},
  number={168},
  pages={1--44},
  year={2024}
}

@inproceedings{ling2025local,
  title={Local causal discovery without causal sufficiency},
  author={Ling, Zhaolong and Yu, Jiale and Zhang, Yiwen and Cheng, Debo and Zhou, Peng and Wu, Xingyu and Jiang, Bingbing and Yu, Kui},
  booktitle={Proceedings of the AAAI Conference on Artificial Intelligence},
  volume={39},
  number={18},
  pages={18737--18745},
  year={2025}
}

@incollection{spirtes2000discovery,
  title={Discovery algorithms without causal sufficiency},
  author={Spirtes, Peter and Glymour, Clark and Scheines, Richard},
  booktitle={Causation, Prediction, and Search},
  pages={163--200},
  year={2000},
  publisher={Springer}
}

@article{xie2023causal,
  title={Causal discovery of 1-factor measurement models in linear latent variable models with arbitrary noise distributions},
  author={Xie, Feng and Zeng, Yan and Chen, Zhengming and He, Yangbo and Geng, Zhi and Zhang, Kun},
  journal={Neurocomputing},
  volume={526},
  pages={48--61},
  year={2023},
  publisher={Elsevier}
}

@inproceedings{markham2020measurement,
  title={Measurement dependence inducing latent causal models},
  author={Markham, Alex and Grosse-Wentrup, Moritz},
  booktitle={Conference on Uncertainty in Artificial Intelligence},
  pages={590--599},
  year={2020},
  organization={PMLR}
}

@article{ying2025generalized,
  title={A generalized tetrad constraint for testing conditional independence given a latent variable},
  author={Ying, Naiwen and Zhang, Ping and Luo, Shanshan and Miao, Wang},
  journal={arXiv preprint arXiv:2504.14173},
  year={2025}
}

@article{scholkopf2021toward,
  title={Toward causal representation learning},
  author={Sch{\"o}lkopf, Bernhard and Locatello, Francesco and Bauer, Stefan and Ke, Nan Rosemary and Kalchbrenner, Nal and Goyal, Anirudh and Bengio, Yoshua},
  journal={Proceedings of the IEEE},
  volume={109},
  number={5},
  pages={612--634},
  year={2021},
  publisher={IEEE}
}

@inproceedings{ahuja2023interventional,
  title={Interventional causal representation learning},
  author={Ahuja, Kartik and Mahajan, Divyat and Wang, Yixin and Bengio, Yoshua},
  booktitle={International conference on machine learning},
  pages={372--407},
  year={2023},
  organization={PMLR}
}

@article{brehmer2022weakly,
  title={Weakly supervised causal representation learning},
  author={Brehmer, Johann and De Haan, Pim and Lippe, Phillip and Cohen, Taco S},
  journal={Advances in Neural Information Processing Systems},
  volume={35},
  pages={38319--38331},
  year={2022}
}

@inproceedings{khemakhem2020variational,
  title={Variational autoencoders and nonlinear {{\textcolor{blue}{ICA}}}: A unifying framework},
  author={Khemakhem, Ilyes and Kingma, Diederik and Monti, Ricardo and Hyvarinen, Aapo},
  booktitle={International conference on artificial intelligence and statistics},
  pages={2207--2217},
  year={2020},
  organization={PMLR}
}

@inproceedings{hyvarinen2019nonlinear,
  title={Nonlinear {{\textcolor{blue}{ICA}}} using auxiliary variables and generalized contrastive learning},
  author={Hyvarinen, Aapo and Sasaki, Hiroaki and Turner, Richard},
  booktitle={The 22nd international conference on artificial intelligence and statistics},
  pages={859--868},
  year={2019},
  organization={PMLR}
}

@article{xu2024sparsity,
  title={A sparsity principle for partially observable causal representation learning},
  author={Xu, Danru and Yao, Dingling and Lachapelle, S{\'e}bastien and Taslakian, Perouz and Von K{\"u}gelgen, Julius and Locatello, Francesco and Magliacane, Sara},
  journal={arXiv preprint arXiv:2403.08335},
  year={2024}
}

@article{lachapelle2024nonparametric,
  title={Nonparametric partial disentanglement via mechanism sparsity: Sparse actions, interventions and sparse temporal dependencies},
  author={Lachapelle, S{\'e}bastien and L{\'o}pez, Pau Rodr{\'\i}guez and Sharma, Yash and Everett, Katie and Priol, R{\'e}mi Le and Lacoste, Alexandre and Lacoste-Julien, Simon},
  journal={arXiv preprint arXiv:2401.04890},
  year={2024}
}

@article{lachapelle2022partial,
  title={Partial disentanglement via mechanism sparsity},
  author={Lachapelle, S{\'e}bastien and Lacoste-Julien, Simon},
  journal={arXiv preprint arXiv:2207.07732},
  year={2022}
}

@article{yao2022temporally,
  title={Temporally disentangled representation learning},
  author={Yao, Weiran and Chen, Guangyi and Zhang, Kun},
  journal={Advances in Neural Information Processing Systems},
  volume={35},
  pages={26492--26503},
  year={2022}
}

@article{varici2023score,
  title={Score-based causal representation learning with interventions},
  author={Varici, Burak and Acarturk, Emre and Shanmugam, Karthikeyan and Kumar, Abhishek and Tajer, Ali},
  journal={arXiv preprint arXiv:2301.08230},
  year={2023}
}

@article{shen2022weakly,
  title={Weakly supervised disentangled generative causal representation learning},
  author={Shen, Xinwei and Liu, Furui and Dong, Hanze and Lian, Qing and Chen, Zhitang and Zhang, Tong},
  journal={Journal of Machine Learning Research},
  volume={23},
  number={241},
  pages={1--55},
  year={2022}
}

@article{locatello2020sober,
  title={A sober look at the unsupervised learning of disentangled representations and their evaluation},
  author={Locatello, Francesco and Bauer, Stefan and Lucic, Mario and R{\"a}tsch, Gunnar and Gelly, Sylvain and Sch{\"o}lkopf, Bernhard and Bachem, Olivier},
  journal={Journal of Machine Learning Research},
  volume={21},
  number={209},
  pages={1--62},
  year={2020}
}

@inproceedings{reddy2022causally,
  title={On causally disentangled representations},
  author={Reddy, Abbavaram Gowtham and Balasubramanian, Vineeth N and others},
  booktitle={Proceedings of the AAAI Conference on Artificial Intelligence},
  volume={36},
  number={7},
  pages={8089--8097},
  year={2022}
}

@incollection{lee1998independent,
  title={Independent component analysis},
  author={Lee, Te-Won},
  booktitle={Independent component analysis: Theory and applications},
  pages={27--66},
  year={1998},
  publisher={Springer}
}

@article{hyvarinen2013independent,
  title={Independent component analysis: recent advances},
  author={Hyv{\"a}rinen, Aapo},
  journal={Philosophical Transactions of the Royal Society A: Mathematical, Physical and Engineering Sciences},
  volume={371},
  number={1984},
  year={2013},
  publisher={The Royal Society}
}

@incollection{hyvarinen2001independent,
  title={Independent component analysis},
  author={Hyv{\"a}rinen, Aapo and Hurri, Jarmo and Hoyer, Patrik O},
  booktitle={Natural image statistics: A probabilistic approach to early computational vision},
  pages={151--175},
  year={2001},
  publisher={Springer}
}

@inproceedings{cai2023causal,
  title={Causal discovery with latent confounders based on higher-order cumulants},
  author={Cai, Ruichu and Huang, Zhiyi and Chen, Wei and Hao, Zhifeng and Zhang, Kun},
  booktitle={International conference on machine learning},
  pages={3380--3407},
  year={2023},
  organization={PMLR}
}

@article{forre2018constraint,
  title={Constraint-based causal discovery for non-linear structural causal models with cycles and latent confounders},
  author={Forr{\'e}, Patrick and Mooij, Joris M},
  journal={arXiv preprint arXiv:1807.03024},
  year={2018}
}

@article{louis1982finding,
  title={Finding the Observed Information Matrix when Using the EM Algorithm},
  author={Louis, Thomas A.},
  journal={Journal of the Royal Statistical Society: Series B (Methodological)},
  volume={44},
  number={2},
  pages={226--233},
  year={1982}
}

\appendix
{\small
\clearpage
\appendix
  \textit{\large Supplement to}\\ \ \\
      % {\Large \bf ``\texorpdfstring{\raisebox{-1mm}{\includegraphics[scale=0.45]{figures/icon.pdf}}}{NSCtrl}: Causal Temporal Representation Learning with Nonstationary Dynamics''}\
      {\Large \bf ``Learning the Structure of Nonlinear Latent Variable Models via Hessian Rank Constraint''}\
\newcommand{\beginsupplement}{%
	\renewcommand{\thetable}{A\arabic{table}}%\renewcommand{\theHtable}{A\arabic{table}}%

        \setcounter{equation}{0}
	\renewcommand{\theequation}{A\arabic{equation}}%\renewcommand{\theHequation}{A\arabic{equation}}%

	\renewcommand{\thefigure}{A\arabic{figure}}%\renewcommand{\theHfigure}{A\arabic{figure}}%
	
	\setcounter{algorithm}{0}
	\renewcommand{\thealgorithm}{A\arabic{algorithm}}%\renewcommand{\theHalgorithm}{A\arabic{algorithm}}%
	
	\setcounter{section}{0}\renewcommand{\theHsection}{A\arabic{section}}%

    \setcounter{theorem}{0}
    \renewcommand{\thetheorem}{A\arabic{theorem}}%\renewcommand{\theHtheorem}{A\arabic{theorem}}%

    \setcounter{corollary}{0}
    \renewcommand{\thecorollary}{A\arabic{corollary}}%\renewcommand{\theHcorollary}{A\arabic{corollary}}%

        \setcounter{lemma}{0}
    \renewcommand{\thelemma}{A\arabic{lemma}}%\renewcommand{\theHlemma}{A\arabic{lemma}}%

    % Add any other environments as needed
}

\beginsupplement
{\large Appendix organization:}

\DoToC 
\clearpage

\section{Related Works}
\label{app:related_works}
\subsection{Causal discovery}
Causal discovery studies how causal relations can be inferred from observational or interventional data, and has been widely investigated in machine learning, statistics, and scientific discovery \citep{glymour2019review,huang2020causal,zhu2019causal,pearl2009causality}. Most classical methods are built on graphical assumptions such as the causal Markov condition \citep{hausman1999independence} and faithfulness \citep{zhalama2017weakening}. Under these assumptions, constraint-based methods use conditional independence tests to recover a Markov equivalence class of DAGs, with PC-style algorithms being representative examples \citep{glymour2019review,spirtes2016causal,zhang2012kernel}. Other methods formulate causal discovery as score-based search or continuous optimization over DAGs \citep{huang2018generalized,ng2024score,montagna2023scalable,li2025strong,yin2024optimization}. A common limitation of many such methods is the causal sufficiency assumption, which excludes unobserved common causes among measured variables \citep{spirtes2000discovery,ling2025local,cai2023causal,forre2018constraint}. In practice, this assumption is often too strong, since latent variables can induce dependencies among observed variables and make the observed causal structure ambiguous. FCI and its variants address this issue by allowing latent confounders and representing the remaining uncertainty over causal structures \citep{spirtes2000causation,pearl2009causality,zhang2008completeness,colombo2012learning,akbari2021recursive,mokhtarian2025recursive}. However, these methods primarily describe causal relations among observed variables. They do not, in general, identify the latent variables themselves or recover causal relations among them, which is the focus of latent-variable causal discovery.

\subsection{Causal discovery among latent variables}

A central goal in latent-variable causal discovery is to recover causal relations among latent variables \citep{huang2022latent,dong2023versatile}. This problem is often studied in measurement-model settings \citep{xie2023causal,markham2020measurement}, where observed variables serve as indicators or descendants of latent causes. Beyond detecting latent confounding, one aims to locate the latent variables, determine their dimensions, and learn their causal structure. Since this structure is not identifiable without additional assumptions, existing methods exploit various statistical signatures. Rank-based methods use tetrad constraints or low-rank cross-covariance matrices to identify pure measurement clusters or infer the number of latent variables that d-separate two measured sets \citep{silva2006learning,kummerfeld2016causal,huang2022latent,xie2022identification,dong2023versatile}. Other methods rely on higher-order moments \citep{shimizu2009estimation,cai2019triad,xie2020generalized,adams2021identification,chen2022identification}, matrix decomposition \citep{anandkumar2013learning,chen2025identification}, copula formulations \citep{cui2018learning}, multi-domain information \citep{zeng2021causal,li2023causal,sturma2023unpaired,zhang2024causal}, or mixture-based oracles \citep{kivva2021learning,kivva2022identifiability}. Despite their identifiability guarantees, many of these methods assume linear measurements, discrete latent variables, or other restrictive structures, which may not hold when observations are generated through nonlinear mechanisms.

Recent work has extended latent causal discovery to nonlinear measurement models. \cite{kong2023identification} study nonlinear latent hierarchical structures but require latent variables to be deterministic functions of measured variables. \cite{prashant2025differentiable} use the Jacobian rank of conditional expectation mappings to identify separating latent sets and learn the structure through differentiable optimization. Recent work has also generalized tetrad constraints to nonlinear models for testing conditional independence given a single latent variable \citep{ying2025generalized}, whereas our goal is to identify latent dimensions, locate latent variables, and recover their causal relations. Our work instead uses the cross-Hessian rank of the observed log-density, providing a nonlinear rank criterion that can be incorporated into existing clustering and PC-style latent discovery procedures.

\subsection{Causal representation learning}

Our work is also related to causal representation learning, which aims to recover latent variables from observed data \citep{scholkopf2021toward,ahuja2023interventional,brehmer2022weakly}. Many causal representation learning methods build on ideas from nonlinear independent component analysis (ICA) \citep{khemakhem2020variational,hyvarinen2019nonlinear}, where identifiability of latent factors is obtained under additional assumptions such as auxiliary variables \citep{khemakhem2020variational,zhang2024causal}, temporal dependence \citep{hyvarinen2019nonlinear}, sparsity \citep{xu2024sparsity,lachapelle2024nonparametric,lachapelle2022partial}, or nonstationarity \citep{yao2022temporally}. Other CRL approaches use interventions \citep{ahuja2023interventional,varici2023score}, multiple environments \citep{zhang2024causal}, or weak supervision \citep{shen2022weakly} to learn causally meaningful representations. These works share our interest in latent variables behind observations, but they mainly focus on representation recovery or disentanglement \citep{locatello2020sober,reddy2022causally}. In contrast, our goal is to locate latent variables and recover causal relations among them from observational data. Moreover, while ICA-based approaches often assume independent latent sources \citep{lee1998independent,hyvarinen2013independent,hyvarinen2001independent}, our setting allows latent variables to be causally dependent.

\section{Notation Description}
\label{app:notation}
This section collects the main notations used in the theorem proofs for clarity and consistency.
\begin{table}[!ht]
\vspace{-2mm}
\caption{List of main notations, explanations, and corresponding values.}
\centering
\scriptsize
\setlength{\tabcolsep}{3pt}
\renewcommand{\arraystretch}{0.86}
\begin{tabular}{p{0.23\textwidth} p{0.43\textwidth} p{0.28\textwidth}}
\toprule
$\textbf{Notation}$ & $\textbf{Explanation}$ & $\textbf{Support / Value}$ \\
\midrule
$\textbf{Notation Style}$ & & \\
\midrule
\textit{Uppercase Roman letters} & random variables or random vectors & \textemdash{} \\
\textit{Lowercase Roman letters} & realizations of the corresponding random variables or random vectors & \textemdash{} \\
\textit{Calligraphic letters} & sets of nodes in the graph & \textemdash{} \\
\textit{Bold uppercase letters} & matrices & \textemdash{} \\
$P$ & probability distributions or probability measures & \textemdash{} \\
$p$ & probability density functions & \textemdash{} \\
\midrule
$\textbf{Graph and Variables}$ & & \\
\midrule
$\mathcal G=(\mathcal V,\mathcal E)$ & directed acyclic graph with node set $\mathcal V$ and edge set $\mathcal E$ & $\mathcal V=\mathcal L\cup\mathcal X$ \\
$\mathcal L=\{L_i\}_{i=1}^{n}$ & set of all latent nodes & $|\mathcal L|=n$ \\
$\mathcal X=\{X_j\}_{j=1}^{m}$ & set of all observed nodes & $|\mathcal X|=m$ \\
$n,m$ & numbers of latent and observed variables, respectively & $n,m\in\mathbb{N}^{+}$ \\
$A,B$ & two disjoint groups of observed variables & $A,B\subseteq\mathcal X$, $A\cap B=\emptyset$ \\
$\mathcal Z;\ Z,z$ & specified latent d-separator; associated random vector and realization & $\mathcal Z\subseteq\mathcal L$, $|\mathcal Z|=d$; $Z,z\in\mathbb R^d$ \\
$d^*$ & cardinality of a minimum set of latent variables that d-separates $A$ and $B$ & $d^*=|\mathcal Z^*|$ \\
$Q$ & conditioning set of latent nodes & $Q\subseteq\mathcal L\setminus\{L_i,L_j\}$, $|Q|=q$ \\
$C_k$ & pure cluster corresponding to latent variable $L_k$ & $C_k\subseteq\mathcal X$ \\
\midrule
$\textbf{Structural Model}$ & & \\
\midrule
$L_i=f_{L_i}(\mathrm{Pa}(L_i),\epsilon_{L_i})$ 
& structural equation for latent variable $L_i$ 
& $\mathrm{Pa}(L_i)\subseteq\mathcal L$ \\
$X_j=f_{X_j}(\mathrm{Pa}(X_j),\epsilon_{X_j})$ 
& structural equation for observed variable $X_j$ 
& $\mathrm{Pa}(X_j)\subseteq\mathcal V$ \\
$\mathrm{Pa}(\cdot)$ & parent set of a variable in $\mathcal{G}$ & / \\
$\epsilon_{L_i},\epsilon_{X_j}$ & independent noise terms of $L_i$ and $X_j$ & / \\
$\begin{gathered}
p_{A,B}(a,b),\\
p_{A,B\mid Z}(a,b\mid z),\\
p_{Z\mid A,B}(z\mid a,b)
\end{gathered}$
& marginal, conditional, and separator-posterior densities
& / \\
\midrule
\textbf{Scores} & & \\
\midrule
$\begin{gathered}
\mathbf s_A(a,z),\\
\mathbf s_B(b,z)
\end{gathered}$
& conditional score functions of $A$ and $B$
& $\in\mathbb{R}^{|A|}$ and $\mathbb{R}^{|B|}$ \\
$\mathbf{u}_A(a),\mathbf{u}_B(b)$
& intercept terms in the Affine Derivatives assumption
& $\mathbf{u}_A\in\mathbb{R}^{|A|}$, $\mathbf{u}_B\in\mathbb{R}^{|B|}$ \\
$\mathbf{V}_A(a),\mathbf{V}_B(b)$
& slope matrices in the Affine Derivatives assumption
& $\mathbf{V}_A\in\mathbb{R}^{|A|\times d}$, $\mathbf{V}_B\in\mathbb{R}^{|B|\times d}$ \\
$\operatorname{Cov}(Z\mid A=a,B=b)$
& posterior covariance of a specified separator
& $\in\mathbb{R}^{d\times d}$ \\
$\mathbb{E}_{Z\mid A=a,B=b}[\cdot]$
& posterior expectation with respect to $P_{Z\mid A=a,B=b}$
& / \\
\midrule
\textbf{HRC} & & \\
\midrule
$\mathbf H_{A,B}$
& cross-Hessian function of the marginal log-density
& Equation~(\ref{eq:cross-hessian}) \\
$\operatorname{rank}(\mathbf H_{A,B})$
& rank of the cross-Hessian function
& largest matrix rank outside a $P_{A,B}$-null set \\
$\operatorname{rank}(\mathbf H_{A,B})=d^*$
& HRC equality under Theorem~\ref{thm:cross-hessian-constraints}
& functional rank equals the minimum separator cardinality \\
$\widehat{\mathbf H}_{A,B}(y)$
& empirical marginal cross-Hessian estimated from $Y$-samples
& Equation~(\ref{eq:kde-cross-hessian}) \\
$E(y)$ 
& cross-Hessian estimation error 
& $\widehat{\mathbf H}_{A,B}(y)-\mathbf H_{A,B}(y)$ \\
$\varepsilon$ 
& high-probability error threshold 
& $\|E(y)\|_2\leq\varepsilon$ \\
$\sigma_j(\cdot),\mathrm{SVD}_{\varepsilon}(\cdot)$ 
& singular values and hard-thresholded SVD operator 
& retain singular values $>\alpha$ \\
\bottomrule
\end{tabular}
\vspace{-2mm}
\end{table}

% \clearpage

\section{Definitions}
\label{def:measurement_model}
\begin{definition}[Measurement Model]
Given a latent variable model $\mathcal G=(\mathcal V,\mathcal E)$ with $\mathcal V=\mathcal L\cup\mathcal X$, the subgraph containing all nodes in $\mathcal V$ and exactly those edges directed into the observed nodes in $\mathcal X$ is called the \textbf{measurement model} of $\mathcal G$ \citep{silva2006learning}.
\end{definition}

\clearpage
\section{Proof}

\subsection{Proof of Theorem 1}
\label{app:the1}

\begin{quote}
\noindent\textbf{Theorem~\ref{thm:latent-dimension} (Cross-Hessian Rank under d-Separation by Latent Variables).}
Let $A,B$ be disjoint observed groups and $\mathcal Z\subseteq\mathcal L$ a set of
latent variables that d-separates them. Let $Z\in\mathbb R^d$ be its associated
random vector, where $d=|\mathcal Z|$. If Assumptions~\ref{assump:regularity}
and~\ref{assump:local-linearity} hold for $Z$, then
\[
\operatorname{rank}\!\left(\mathbf H_{A,B}\right)\le d.
\]
\end{quote}

\begin{proof}
\textcolor{blue}{
Fix a set $\mathcal Z$ as in the theorem and a point
$y=(y_A^\top,y_B^\top)^\top$ at which $p_{A,B}(y)>0$ and all derivatives
below exist. Write
\[
q_y(z):=p_{Z\mid A,B}(z\mid y_A,y_B)
=\frac{p_{A,B\mid Z}(y\mid z)p_Z(z)}{p_{A,B}(y)}
\]
for the posterior density of $Z$ at $y$. Assumption~\ref{assump:regularity}
justifies each interchange of differentiation and integration used below.
Because the structural causal model is Markov with respect to $\mathcal G$
and $\mathcal Z$ d-separates $A$ and $B$, we have
$A\perp\!\!\!\perp B\mid Z$. Hence
\begin{equation}
p_{A,B\mid Z}(y_A,y_B\mid z)
=p_{A\mid Z}(y_A\mid z)p_{B\mid Z}(y_B\mid z).
\label{eq:appendix-conditional-factorization}
\end{equation}
Let
$\mathbf s_A(y_A,z)=\nabla_{y_A}\log p_{A\mid Z}(y_A\mid z)$ and
$\mathbf s_B(y_B,z)=\nabla_{y_B}\log p_{B\mid Z}(y_B\mid z)$.
Differentiating the marginalization identity
\[
p_{A,B}(y)=\int p_{A,B\mid Z}(y\mid z)p_Z(z)\,\mathrm dz
\]
with respect to $y_A$ gives
\begin{align}
\nabla_{y_A}p_{A,B}(y)
&=\int \nabla_{y_A}p_{A,B\mid Z}(y\mid z)p_Z(z)\,\mathrm dz \notag\\
&=\int p_{A,B\mid Z}(y\mid z)p_Z(z)\,
\mathbf s_A(y_A,z)\,\mathrm dz.
\label{eq:appendix-first-density-derivative}
\end{align}
The analogous identity holds for the derivative with respect to $y_B$.
Differentiating once more yields
\begin{align}
\nabla_{y_A}\nabla_{y_B}^\top p_{A,B}(y)
={}&\int p_{A,B\mid Z}(y\mid z)p_Z(z)
\Bigl[
\mathbf s_A(y_A,z)\mathbf s_B(y_B,z)^\top \notag\\
&\hspace{32mm}+
\nabla_{y_A}\nabla_{y_B}^\top
\log p_{A,B\mid Z}(y\mid z)
\Bigr] \,\mathrm dz.
\label{eq:appendix-second-density-derivative}
\end{align}
By Equation~(\ref{eq:appendix-conditional-factorization}),
$\log p_{A,B\mid Z}=\log p_{A\mid Z}+\log p_{B\mid Z}$, so the mixed
conditional log-density derivative in
Equation~(\ref{eq:appendix-second-density-derivative}) is zero.
For any positive twice differentiable density $p$,
\[
\nabla_{y_A}\nabla_{y_B}^\top\log p
=
\frac{\nabla_{y_A}\nabla_{y_B}^\top p}{p}
-
\frac{(\nabla_{y_A}p)(\nabla_{y_B}p)^\top}{p^2}.
\]
Substituting Equations~(\ref{eq:appendix-first-density-derivative}) and
(\ref{eq:appendix-second-density-derivative}) into this identity and using
Bayes' rule gives
\begin{align}
\mathbf H_{A,B}(y)
={}&\mathbb E_{q_y}
\!\left[\mathbf s_A(y_A,Z)\mathbf s_B(y_B,Z)^\top\right] \notag\\
&-\mathbb E_{q_y}[\mathbf s_A(y_A,Z)]
\mathbb E_{q_y}[\mathbf s_B(y_B,Z)]^\top \notag\\
={}&\operatorname{Cov}_{q_y}
\!\left(\mathbf s_A(y_A,Z),\mathbf s_B(y_B,Z)\right).
\label{eq:appendix-score-covariance}
\end{align}
Under Assumption~\ref{assump:local-linearity},
\[
\mathbf s_A(y_A,Z)=\mathbf u_A(y_A)+\mathbf V_A(y_A)Z,
\qquad
\mathbf s_B(y_B,Z)=\mathbf u_B(y_B)+\mathbf V_B(y_B)Z.
\]
Let $\boldsymbol\mu_y=\mathbb E_{q_y}[Z]$. Centering the two scores under
$q_y$ removes the intercepts:
\[
\mathbf s_A-\mathbb E_{q_y}[\mathbf s_A]
=\mathbf V_A(y_A)(Z-\boldsymbol\mu_y),
\qquad
\mathbf s_B-\mathbb E_{q_y}[\mathbf s_B]
=\mathbf V_B(y_B)(Z-\boldsymbol\mu_y).
\]
Therefore Equation~(\ref{eq:appendix-score-covariance}) becomes
\begin{align}
\mathbf H_{A,B}(y)
&=\mathbb E_{q_y}\!\left[
\mathbf V_A(y_A)(Z-\boldsymbol\mu_y)(Z-\boldsymbol\mu_y)^\top
\mathbf V_B(y_B)^\top\right] \notag\\
&=\mathbf V_A(y_A)\,
\operatorname{Cov}(Z\mid A=y_A,B=y_B)\,
\mathbf V_B(y_B)^\top.
\label{eq:appendix-hessian-factorization}
\end{align}
The factors in Equation~(\ref{eq:appendix-hessian-factorization}) have
sizes $|A|\times d$, $d\times d$, and $d\times|B|$. Consequently,
\[
\operatorname{rank}(\mathbf H_{A,B}(y))
\leq
\min\!\left\{
\operatorname{rank}(\mathbf V_A(y_A)),
\operatorname{rank}(\operatorname{Cov}(Z\mid A=y_A,B=y_B)),
\operatorname{rank}(\mathbf V_B(y_B))
\right\}
\leq d.
\]
Thus the pointwise matrix rank is at most $d$ wherever the cross-Hessian is defined. Consequently, the rank of the matrix-valued function satisfies $\operatorname{rank}(\mathbf H_{A,B})\leq d$.
}
\end{proof}

\paragraph{\textcolor{revisiondarkgreen}{Relation to the linear Gaussian case.}}

For non-degenerate jointly Gaussian distributions, the block-inverse calculation in the proof of Corollary~\ref{coro1} shows that cross-Hessian rank equals cross-covariance rank. Classical algebraic results then imply that, in linear Gaussian structural equation models, additional rank loss occurs only on exceptional measure-zero parameter subsets \citep{Sullivant-T-separation}. This comparison motivates the explicit non-degeneracy condition used for exact rank recovery; no genericity claim is made for unrestricted nonlinear models.

\subsection{Proof of Corollary 2.1}
\label{app:cor1}
\begin{quote}
\noindent\textbf{Corollary~\ref{coro1} (Linear Gaussian Special Case).}
Let $A$ and $B$ be disjoint observed groups. Let $\mathcal Z^*$ be a
minimum latent d-separator between them, write $d^*=|\mathcal Z^*|$, and
let $Z^*\in\mathbb R^{d^*}$ be its associated random vector. Suppose the
full structural causal model is linear Gaussian, the joint distribution of
$(A,B)$ is non-degenerate, and Assumptions~\ref{assump:regularity}
and~\ref{assump:nondegeneracy} hold. Then
\[
\operatorname{rank}\!\left(\mathbf H_{A,B}\right)=d^*.
\]
\end{quote}

\begin{proof}
Under the stated assumptions, $Y=(A^\top,B^\top)^\top$ is a
non-degenerate jointly Gaussian vector. Let $\Sigma$ and
$\Omega=\Sigma^{-1}$ denote its covariance and precision matrices,
respectively. Since the Gaussian log-density has Hessian $-\Omega$, its
cross-Hessian block is $\mathbf H_{A,B}=-\Omega_{A,B}$. The block-inverse
identity gives
\[
\mathbf H_{A,B}
=
\Sigma_{A,A}^{-1}\Sigma_{A,B}
\bigl(\Sigma_{B,B}-\Sigma_{B,A}\Sigma_{A,A}^{-1}\Sigma_{A,B}\bigr)^{-1}.
\]
Because $\Sigma$ is positive definite, the two matrices multiplying
$\Sigma_{A,B}$ are invertible, and hence
\[
\operatorname{rank}(\mathbf H_{A,B})
=
\operatorname{rank}(\Sigma_{A,B}).
\]
This establishes the equivalence between covariance-rank and
cross-Hessian-rank constraints in the non-degenerate Gaussian case. Writing
$A=\Lambda_A Z^*+\epsilon_A$ and $B=\Lambda_B Z^*+\epsilon_B$, the
conditional scores are affine in $Z^*$, with
$\mathbf V_A=\Sigma_A^{-1}\Lambda_A$ and
$\mathbf V_B=\Sigma_B^{-1}\Lambda_B$. Thus
\[
\mathbf H_{A,B}
=
\mathbf V_A\,
\operatorname{Cov}(Z^*\mid A,B)\,
\mathbf V_B^\top.
\]
Assumption~\ref{assump:nondegeneracy} makes the middle factor positive
definite almost surely and the two slope matrices full column rank $d^*$.
Therefore the cross-Hessian function has rank $d^*$.
\end{proof}

\subsection{Proof of Theorem 2}
\label{app:the2}
\begin{quote}
\noindent\textbf{Theorem~\ref{thm:cross-hessian-constraints} (Hessian Rank Constraint).}
Let $A$ and $B$ be disjoint groups of observed variables, and let
$\mathcal Z^*$ be a minimum-cardinality set of latent variables that
d-separates them. Write $d^*:=|\mathcal Z^*|$. Under Assumptions~
\ref{assump:regularity} and~\ref{assump:local-linearity} with respect to
$\mathcal Z^*$, and under Assumption~\ref{assump:nondegeneracy}, we have
\[
\operatorname{rank}\!\left(\mathbf H_{A,B}\right)=d^*.
\]
\end{quote}

\begin{proof}
Let $Z^*\in\mathbb R^{d^*}$ be the random vector associated with the
minimum separator $\mathcal Z^*$. Equation~(\ref{equ:hessian_decompose})
gives the function identity
\[
\mathbf H_{A,B}
=
\mathbf V_A(A)\,
\operatorname{Cov}(Z^*\mid A,B)\,
\mathbf V_B(B)^\top.
\]
By Assumption~\ref{assump:nondegeneracy}, the posterior covariance is
positive definite almost surely, while $\mathbf V_A(A)$ and $\mathbf V_B(B)$
have full column rank $d^*$. Hence the product has matrix rank $d^*$ outside
a $P_{A,B}$-null set. By the definition of functional rank,
$\operatorname{rank}(\mathbf H_{A,B})=d^*$.
\end{proof}

\subsection{Derivation of the KDE Cross-Hessian Formula}\label{app:kde-cross-hessian-derivation}\textcolor{blue}{Write $y=(y_A^\top,y_B^\top)^\top$ and let $q(y):=\widehat p_h(y)$. Fix an evaluation point at which $q(y)>0$. For coordinates $y_{A,i}$ and $y_{B,j}$,\[\frac{\partial\log q(y)}{\partial y_{B,j}}=\frac{1}{q(y)}\frac{\partial q(y)}{\partial y_{B,j}}.\]Differentiating this identity with respect to $y_{A,i}$ and applying the product rule gives\[\begin{aligned}\frac{\partial^2\log q(y)}{\partial y_{A,i}\,\partial y_{B,j}}&=\frac{\partial}{\partial y_{A,i}}\left[q(y)^{-1}\frac{\partial q(y)}{\partial y_{B,j}}\right]\\&=q(y)^{-1}\frac{\partial^2q(y)}{\partial y_{A,i}\,\partial y_{B,j}}-q(y)^{-2}\frac{\partial q(y)}{\partial y_{A,i}}\frac{\partial q(y)}{\partial y_{B,j}}.\end{aligned}\]The $(i,j)$ entry of $\nabla_{y_A}\nabla_{y_B}^\top\log q(y)$ is the left-hand side. Stacking the identities over all $i$ and $j$ therefore yields\[\nabla_{y_A}\nabla_{y_B}^\top\log q(y)=\frac{\nabla^2_{y_A,y_B}q(y)}{q(y)}-\frac{\nabla_{y_A}q(y)\,\nabla_{y_B}q(y)^\top}{q(y)^2}.\]To make the dependence on the KDE explicit, define $u^{(n)}=(y-Y^{(n)})/h$. If the kernel is twice differentiable, differentiation of the finite sum gives\[\begin{aligned}\nabla_{y_A}\widehat p_h(y)&=\frac{1}{Nh^{s+1}}\sum_{n=1}^N\nabla_{y_A}\mathcal K(u^{(n)}),\\\nabla_{y_B}\widehat p_h(y)&=\frac{1}{Nh^{s+1}}\sum_{n=1}^N\nabla_{y_B}\mathcal K(u^{(n)}),\\\nabla^2_{y_A,y_B}\widehat p_h(y)&=\frac{1}{Nh^{s+2}}\sum_{n=1}^N\nabla^2_{y_A,y_B}\mathcal K(u^{(n)}),\end{aligned}\]where the derivatives on the right-hand side are taken with respect to the corresponding coordinate blocks of the kernel argument. Substituting these derivatives into the matrix identity above with $q=\widehat p_h$ gives Equation~(\ref{eq:kde-cross-hessian}).}

\subsection{Proof of Theorem 3}\label{app:the3}

\begin{lemma}[Taylor remainder under Hessian Lipschitz continuity]
\label{lem:conditional-score-taylor}
Under the conditions of Theorem~\ref{the:error_bound}, let
$\mathbf v_{G,i}^\top$ be the $i$th row of $\mathbf V_G$, define
\begin{equation}
r_{G,i}(z):=s_{G,i}(y_G,z)-s_{G,i}(y_G,z_0)
-\mathbf v_{G,i}^\top(z-z_0),
\label{eq:the3-component-remainder-definition}
\end{equation}
and write
$\widetilde C_{G,i}:=\kappa_{G,i}+\delta M/3$.
Then, for every $z\in\mathcal B(z_0,\delta)$,
\begin{equation}
|r_{G,i}(z)|\leq\frac12\widetilde C_{G,i}\|z-z_0\|_2^2.
\label{eq:the3-component-remainder-bound}
\end{equation}
\end{lemma}

\begin{proof}
Fix $G$ and $i$, and write $h=z-z_0$. Taylor's theorem with integral
remainder gives
\begin{equation}
r_{G,i}(z)=\int_0^1(1-t)h^\top
\nabla_z^2s_{G,i}(y_G,z_0+th)h\,\mathrm dt.
\label{eq:appendix-taylor-integral-remainder}
\end{equation}
Because $\mathcal B(z_0,\delta)$ is convex, the entire segment
$z_0+th$, $t\in[0,1]$, remains in the region covered by the Hessian
Lipschitz condition. Hence
\begin{align}
\|\nabla_z^2s_{G,i}(y_G,z_0+th)\|_2
&\leq \|\nabla_z^2s_{G,i}(y_G,z_0)\|_2 \notag\\
&\quad+\|\nabla_z^2s_{G,i}(y_G,z_0+th)
-\nabla_z^2s_{G,i}(y_G,z_0)\|_2 \notag\\
&\leq \kappa_{G,i}+tM\|h\|_2.
\label{eq:appendix-hessian-lipschitz-bound}
\end{align}
Substitution into Equation~(\ref{eq:appendix-taylor-integral-remainder}) gives
\begin{align}
|r_{G,i}(z)|
&\leq \|h\|_2^2\int_0^1(1-t)
\bigl(\kappa_{G,i}+tM\|h\|_2\bigr)\,\mathrm dt \notag\\
&=\frac12\kappa_{G,i}\|h\|_2^2
+\frac16M\|h\|_2^3 \notag\\
&\leq\frac12\left(\kappa_{G,i}
+\frac{\delta M}{3}\right)\|h\|_2^2
=\frac12\widetilde C_{G,i}\|h\|_2^2,
\end{align}
where the last inequality uses $\|h\|_2\leq\delta$.
\end{proof}

\begin{lemma}[Covariance bound on a posterior ball]
\label{lem:posterior-ball-covariance}
If a random vector $Z$ is supported on $\mathcal B(z_0,\delta)$, then
\begin{equation}
\|\operatorname{Cov}(Z)\|_2\leq\delta^2.
\label{eq:posterior-ball-covariance}
\end{equation}
Equivalently, $\operatorname{Var}(d^\top Z)\leq\delta^2$ for every unit
vector $d$.
\end{lemma}

\begin{proof}
For any unit vector $d$, the scalar projection $d^\top Z$ lies in an
interval of length at most $2\delta$. Indeed, for any
$z_1,z_2\in\mathcal B(z_0,\delta)$,
\begin{equation}
|d^\top z_1-d^\top z_2|
\leq\|d\|_2\|z_1-z_2\|_2\leq2\delta.
\end{equation}
For a scalar random variable $W$ supported on an interval $[m,M]$,
Popoviciu's inequality gives
$\operatorname{Var}(W)\leq(M-m)^2/4$. Hence
$\operatorname{Var}(d^\top Z)\leq\delta^2$. Since
$\operatorname{Cov}(Z)$ is positive semidefinite, the Rayleigh--Ritz
characterization gives
\begin{equation}
\|\operatorname{Cov}(Z)\|_2
=\sup_{\|d\|_2=1}d^\top\operatorname{Cov}(Z)d
=\sup_{\|d\|_2=1}\operatorname{Var}(d^\top Z)
\leq\delta^2.
\end{equation}
\end{proof}

\begin{lemma}[Spectral and Frobenius norms]
\label{lem:spectral-frobenius}
For every matrix $M$, $\|M\|_2\leq\|M\|_F$.
\end{lemma}

\begin{proof}
Let $\sigma_1(M)\geq\sigma_2(M)\geq\cdots\geq0$ be the singular values of
$M$. Then
\begin{equation}
\|M\|_2=\sigma_1(M)
\leq\left(\sum_k\sigma_k(M)^2\right)^{1/2}
=\|M\|_F.
\end{equation}
\end{proof}

\begin{proof}[Proof of Theorem~\ref{the:error_bound}]
The proof consists of three steps.

\textbf{Step 1: Decompose the cross-Hessian approximation error.}
Let $\ell(y,z):=\log p_{A,B\mid Z}(y\mid z)$ and
$Q=P_{Z\mid A=y_A,B=y_B}$. Assumption~\ref{assump:regularity} permits
differentiation through the latent marginalization. First,
\begin{align}
\nabla_{y_A}\log p_{A,B}(y)
&=\frac{1}{p_{A,B}(y)}
\int \nabla_{y_A}p_{A,B\mid Z}(y\mid z)p_Z(z)\,\mathrm dz \notag\\
&=\int \nabla_{y_A}\ell(y,z)
\frac{p_{A,B\mid Z}(y\mid z)p_Z(z)}{p_{A,B}(y)}\,\mathrm dz \notag\\
&=\mathbb E_Q[\nabla_{y_A}\ell(y,Z)].
\label{eq:appendix-the3-fisher-identity}
\end{align}
Differentiating this identity with respect to $y_B$ gives the second-order
Fisher identity
\begin{align}
\mathbf H_{A,B}(y)
={}&\mathbb E_Q[\nabla_{y_A}\nabla_{y_B}^\top\ell(y,Z)] \notag\\
&+\operatorname{Cov}_Q\!\left(
\nabla_{y_A}\ell(y,Z),\nabla_{y_B}\ell(y,Z)\right).
\label{eq:appendix-the3-second-fisher-identity}
\end{align}
Because $Z$ d-separates $A$ and $B$,
$\ell(y,z)=\log p_{A\mid Z}(y_A\mid z)+\log p_{B\mid Z}(y_B\mid z)$.
The mixed derivative in Equation~(\ref{eq:appendix-the3-second-fisher-identity})
is therefore zero, while the two gradients are the conditional scores. Thus
\begin{equation}
\mathbf H_{A,B}(y)
=\operatorname{Cov}_Q\!\left(
\mathbf s_A(y_A,Z),\mathbf s_B(y_B,Z)\right).
\label{eq:score-covariance-approximate}
\end{equation}

To expand this covariance componentwise, let $\mathbf v_{G,i}^\top$ denote
the $i$th row of $\mathbf V_G$. Lemma~\ref{lem:conditional-score-taylor} gives, for $i\in G$,
\begin{equation}
s_{G,i}(y_G,Z)
=s_{G,i}(y_G,z_0)+\mathbf v_{G,i}^\top(Z-z_0)+r_{G,i}(Z).
\label{eq:appendix-the3-component-score-decomposition}
\end{equation}
For $i\in A$ and $j\in B$, constants vanish under covariance, so
Equations~(\ref{eq:score-covariance-approximate}) and
(\ref{eq:appendix-the3-component-score-decomposition}) imply
\begin{align}
H_{ij}(y)
={}&\mathbf v_{A,i}^\top\operatorname{Cov}_Q(Z)\mathbf v_{B,j}
+\mathbf v_{A,i}^\top\operatorname{Cov}_Q(Z,r_{B,j}) \notag\\
&+\operatorname{Cov}_Q(r_{A,i},Z)\mathbf v_{B,j}
+\operatorname{Cov}_Q(r_{A,i},r_{B,j}).
\label{eq:appendix-the3-entrywise-decomposition}
\end{align}
Stacking Equation~(\ref{eq:appendix-the3-entrywise-decomposition}) over all
$i\in A$ and $j\in B$ yields
\begin{align}
\mathbf H_{A,B}(y)
={}&\mathbf V_A\operatorname{Cov}_Q(Z)\mathbf V_B^\top
+\mathbf V_A\operatorname{Cov}_Q(Z,\mathbf r_B) \notag\\
&+\operatorname{Cov}_Q(\mathbf r_A,Z)\mathbf V_B^\top
+\operatorname{Cov}_Q(\mathbf r_A,\mathbf r_B).
\label{eq:appendix-the3-matrix-decomposition}
\end{align}
Under Assumption~\ref{assump:local-linearity}, this first term is the cross-Hessian factorization
$\mathbf H^{\mathrm{aff}}_{A,B}(y):=\mathbf V_A\operatorname{Cov}_Q(Z)\mathbf V_B^\top$. Therefore the
approximation error
$\mathbf E:=\mathbf H_{A,B}(y)-\mathbf H^{\mathrm{aff}}_{A,B}(y)$ is exactly
\begin{equation}
\mathbf E
=\mathbf V_A\operatorname{Cov}_Q(Z,\mathbf r_B)
+\operatorname{Cov}_Q(\mathbf r_A,Z)\mathbf V_B^\top
+\operatorname{Cov}_Q(\mathbf r_A,\mathbf r_B).
\label{eq:appendix-the3-error-decomposition}
\end{equation}

\textbf{Step 2: Bound the nonlinear Taylor terms.}
Let $\widetilde C_{G,i}$ be the curvature constants in Theorem~\ref{the:error_bound}.
By Lemma~\ref{lem:conditional-score-taylor},
\begin{equation}
|r_{G,i}(Z)|\leq\frac12\widetilde C_{G,i}\delta^2
\qquad Q\text{-almost surely}.
\label{eq:appendix-the3-uniform-component-bound}
\end{equation}
The variance of each remainder component therefore satisfies
\begin{align}
\operatorname{Var}_Q(r_{G,i})
&=\mathbb E_Q[(r_{G,i}-\mathbb E_Qr_{G,i})^2] \notag\\
&\leq\mathbb E_Q[r_{G,i}^2]
\leq\frac14\widetilde C_{G,i}^2\delta^4.
\label{eq:appendix-the3-remainder-variance}
\end{align}
For $i\in A$ and $j\in B$, scalar Cauchy--Schwarz gives
\begin{align}
|\operatorname{Cov}_Q(r_{A,i},r_{B,j})|
&\leq
\sqrt{\operatorname{Var}_Q(r_{A,i})
      \operatorname{Var}_Q(r_{B,j})} \notag\\
&\leq\frac14\widetilde C_{A,i}\widetilde C_{B,j}\delta^4.
\label{eq:appendix-the3-remainder-cross-covariance}
\end{align}
Using Lemma~\ref{lem:spectral-frobenius} and summing the squared entrywise
bounds gives
\begin{align}
\|\operatorname{Cov}_Q(\mathbf r_A,\mathbf r_B)\|_2
&\leq\|\operatorname{Cov}_Q(\mathbf r_A,\mathbf r_B)\|_F \notag\\
&\leq\frac{\delta^4}{4}
\left(\sum_{i\in A}\widetilde C_{A,i}^2\right)^{1/2}
\left(\sum_{j\in B}\widetilde C_{B,j}^2\right)^{1/2} \notag\\
&=\frac14\|\widetilde{\mathbf C}_A\|_2\|\widetilde{\mathbf C}_B\|_2\delta^4.
\label{eq:appendix-the3-remainder-matrix-bound}
\end{align}

It remains to control the two latent--remainder covariance matrices. For a
unit vector $d\in\mathbb R^d$, scalar Cauchy--Schwarz,
Lemma~\ref{lem:posterior-ball-covariance}, and
Equation~(\ref{eq:appendix-the3-remainder-variance}) give
\begin{align}
|d^\top\operatorname{Cov}_Q(Z,r_{G,i})|
&=|\operatorname{Cov}_Q(d^\top Z,r_{G,i})| \notag\\
&\leq
\sqrt{\operatorname{Var}_Q(d^\top Z)
      \operatorname{Var}_Q(r_{G,i})} \notag\\
&\leq \delta\left(\frac12\widetilde C_{G,i}\delta^2\right)
=\frac12\widetilde C_{G,i}\delta^3.
\end{align}
Taking the supremum over unit $d$ yields
\begin{equation}
\|\operatorname{Cov}_Q(Z,r_{G,i})\|_2
\leq\frac12\widetilde C_{G,i}\delta^3.
\label{eq:appendix-the3-latent-component-covariance}
\end{equation}
The columns of $\operatorname{Cov}_Q(Z,\mathbf r_B)$ are
$\operatorname{Cov}_Q(Z,r_{B,j})$. Hence, again using
Lemma~\ref{lem:spectral-frobenius},
\begin{align}
\|\operatorname{Cov}_Q(Z,\mathbf r_B)\|_2
&\leq\|\operatorname{Cov}_Q(Z,\mathbf r_B)\|_F \notag\\
&=\left(\sum_{j\in B}
\|\operatorname{Cov}_Q(Z,r_{B,j})\|_2^2\right)^{1/2} \notag\\
&\leq\frac12\delta^3
\left(\sum_{j\in B}\widetilde C_{B,j}^2\right)^{1/2}
=\frac12\|\widetilde{\mathbf C}_B\|_2\delta^3.
\label{eq:appendix-the3-latent-B-bound}
\end{align}
The same argument applied to the rows of
$\operatorname{Cov}_Q(\mathbf r_A,Z)$ gives
\begin{equation}
\|\operatorname{Cov}_Q(\mathbf r_A,Z)\|_2
\leq\frac12\|\widetilde{\mathbf C}_A\|_2\delta^3.
\label{eq:appendix-the3-A-latent-bound}
\end{equation}

\textbf{Step 3: Bound the total approximation error.}
Applying the triangle inequality and submultiplicativity of the spectral
norm to Equation~(\ref{eq:appendix-the3-error-decomposition}) gives
\begin{align}
\|\mathbf E\|_2
\leq{}&\|\mathbf V_A\|_2
\|\operatorname{Cov}_Q(Z,\mathbf r_B)\|_2 \notag\\
&+\|\operatorname{Cov}_Q(\mathbf r_A,Z)\|_2
\|\mathbf V_B\|_2
+\|\operatorname{Cov}_Q(\mathbf r_A,\mathbf r_B)\|_2.
\label{eq:appendix-the3-total-error-prebound}
\end{align}
Substituting Equations~(\ref{eq:appendix-the3-remainder-matrix-bound}),
(\ref{eq:appendix-the3-latent-B-bound}), and
(\ref{eq:appendix-the3-A-latent-bound}) into
Equation~(\ref{eq:appendix-the3-total-error-prebound}) yields
\begin{align}
\|\mathbf H_{A,B}(y)-\mathbf H^{\mathrm{aff}}_{A,B}(y)\|_2
\leq{}&\frac12\delta^3
\left(\|\mathbf V_A\|_2\|\widetilde{\mathbf C}_B\|_2+\|\mathbf V_B\|_2\|\widetilde{\mathbf C}_A\|_2\right) \notag\\
&+\frac14\delta^4\|\widetilde{\mathbf C}_A\|_2\|\widetilde{\mathbf C}_B\|_2,
\end{align}
which is Equation~(\ref{eq:approximate-affine-cross-hessian-bound}).
\end{proof}

\subsection{KDE Cross-Hessian Estimation and Rank Recovery}
\label{app:kde-rank-recovery}
\label{app:kde-cross-hessian-error}
\begin{theorem}[\textbf{Pointwise KDE Cross-Hessian Estimation Error}]
\label{the:kde-cross-hessian-error}
Fix an interior point $y=(y_A^\top,y_B^\top)^\top\in\mathbb R^s$. Suppose
$p_{A,B}$ is bounded away from zero and has bounded continuous derivatives
up to order four in a neighborhood of $y$. Let $\mathcal K$ be a twice
differentiable second-order kernel whose derivatives up to order two are
bounded and square-integrable. Then, for a deterministic bandwidth $h$, with
probability at least $1-\eta$,
{\small
\begin{equation}
\left\|\widehat{\mathbf H}_{A,B}(y)-\mathbf H_{A,B}(y)\right\|_2
\leq C\left\{h^2+\sqrt{\frac{\log(C_0/\eta)}{Nh^{s+4}}}\right\}.
\label{eq:kde-cross-hessian-error}
\end{equation}
}
provided $Nh^s\geq\log(C_0/\eta)$, where $C$ and $C_0$ depend only on
the local density bounds, the kernel, and the fixed dimensions.
\end{theorem}
\begin{proof}
\textcolor{blue}{Let $p=p_{A,B}$, and let $\gamma$ be a multi-index of
order $r=|\gamma|\leq2$. Differentiating the finite KDE sum gives
{\small
\begin{equation}
D^\gamma\widehat p_h(y)
=\frac{1}{Nh^{s+r}}\sum_{i=1}^N
D^\gamma\mathcal K\!\left(\frac{y-Y^{(i)}}{h}\right).
\label{eq:kde-derivative-estimator}
\end{equation}
}
Because $\mathcal K$ is a second-order kernel and $p$ has bounded continuous
derivatives up to order four near $y$, the standard kernel bias expansion
yields
{\small
\begin{equation}
\left|\mathbb E[D^\gamma\widehat p_h(y)]-D^\gamma p(y)\right|
\leq B_\gamma h^2.
\label{eq:kde-derivative-bias}
\end{equation}
}
The boundedness and square integrability of $D^\gamma\mathcal K$, followed by
Bernstein's inequality, imply that, with probability at least $1-\eta$,
{\small
\begin{equation}
\left|D^\gamma\widehat p_h(y)-
\mathbb E[D^\gamma\widehat p_h(y)]\right|
\leq C_\gamma\!\left\{
\sqrt{\frac{\log(C_0/\eta)}{Nh^{s+2r}}}
+\frac{\log(C_0/\eta)}{Nh^{s+r}}\right\}.
\label{eq:kde-derivative-concentration}
\end{equation}
}
A union bound makes these inequalities hold simultaneously for the density,
all first derivatives, and all second derivatives. Under
$Nh^s\geq\log(C_0/\eta)$, the second stochastic term is dominated by the
square-root term. Since $r\leq2$ and $h\leq1$, the common upper bound is
{\small
\begin{equation}
\Delta_{N,h,\eta}
:=C'\left\{h^2+
\sqrt{\frac{\log(C_0/\eta)}{Nh^{s+4}}}\right\}.
\label{eq:kde-common-derivative-bound}
\end{equation}
}}

\textcolor{blue}{Write $p_0=p(y)$, $\mathbf d_A=\nabla_{y_A} p(y)$,
$\mathbf d_B=\nabla_{y_B} p(y)$, and $M_{A,B}=\nabla^2_{y_A,y_B}p(y)$. Define
{\small
\begin{equation}
\Phi(p_0,\mathbf d_A,\mathbf d_B,M_{A,B})
:=\frac{M_{A,B}}{p_0}-\frac{\mathbf d_A\mathbf d_B^\top}{p_0^2}.
\label{eq:log-hessian-map}
\end{equation}
}}

\textcolor{blue}{
Then $\mathbf H_{A,B}(y)=\Phi(p_0,\mathbf d_A,\mathbf d_B,M_{A,B})$, and the analogous
identity holds for the KDE quantities. Because $p$ is bounded away from zero
near $y$, the event in Equation~(\ref{eq:kde-common-derivative-bound}) also
ensures $\widehat p_h(y)$ is bounded away from zero for sufficiently small
$h$ and sufficiently large $N$. On this region, $\Phi$ is Lipschitz in all
four arguments, with a constant depending only on the local density and
derivative bounds. Consequently,
{\small
\begin{equation}
\left\|\widehat{\mathbf H}_{A,B}(y)-
\mathbf H_{A,B}(y)\right\|_2
\leq C\left\{h^2+
\sqrt{\frac{\log(C_0/\eta)}{Nh^{s+4}}}\right\}.
\end{equation}
}
Absorbing the fixed matrix dimensions and the preceding Lipschitz constant
into $C$ proves Equation~(\ref{eq:kde-cross-hessian-error}).}
\end{proof}

\begin{theorem}[\textbf{Finite-Sample Cross-Hessian Rank Recovery}]
\label{the:empirical_rank_id}
Under the assumptions of Theorem~\ref{thm:cross-hessian-constraints}, fix an
interior point $y$ at which $\operatorname{rank}(\mathbf H_{A,B}(y))=d^*$.
Let $\alpha$ satisfy, with probability at least $1-\eta$,
$\|\widehat{\mathbf H}_{A,B}(y)-\mathbf H_{A,B}(y)\|_2\leq\alpha$.
If $\sigma_{d^*}(\mathbf H_{A,B}(y))>2\alpha$, then, with probability at
least $1-\eta$,
{\small
\begin{equation}
\operatorname{rank}\!\left(\operatorname{SVD}_{\alpha}
(\widehat{\mathbf H}_{A,B}(y))\right)=d^*,
\label{eq:finite-confidence-rank}
\end{equation}
}
where $\operatorname{SVD}_{\alpha}$ retains only singular values strictly
larger than $\alpha$.
\end{theorem}
\label{app:the4}
\begin{proof}
\textcolor{blue}{On the event in Theorem~\ref{the:empirical_rank_id},
Weyl's inequality gives
\[
\left|\sigma_j(\widehat{\mathbf H}_{A,B}(y))
-\sigma_j(\mathbf H_{A,B}(y))\right|\leq\alpha
\]
for every $j$. By the premise of the theorem,
$\mathbf H_{A,B}(y)$ has rank $d^*$, so every singular value of
$\widehat{\mathbf H}_{A,B}(y)$ with $j>d^*$ is at most $\alpha$.
Moreover,
\[
\sigma_{d^*}(\widehat{\mathbf H}_{A,B}(y))
\geq\sigma_{d^*}(\mathbf H_{A,B}(y))-\alpha>\alpha.
\]
Thus exactly the first $d^*$ singular values remain after hard thresholding
at $\alpha$.}
\end{proof}

\paragraph{Scope of the bootstrap calibration.}
\textcolor{blue}{Theorem~\ref{the:empirical_rank_id} is conditional on a valid
pointwise spectral-norm bound for the observed-data estimator. The bootstrap in
the main text provides an operational numerical calibration of this estimation
error; it is not invoked in either proof, and no bootstrap rank-consistency
claim is made.}

\subsection{Evaluation Points and Rank Aggregation}\label{app:rank-aggregation}\textcolor{blue}{For each queried pair $(A,B)$, we select the KDE bandwidth $h$ once by leave-one-out likelihood cross-validation. We begin with a prespecified finite candidate set of interior points, discard candidates whose fitted density is below a prespecified lower quantile, and denote the retained set by $\mathcal Y_{\mathrm{eval}}=\{y_1,\ldots,y_M\}$. After this selection, both $h$ and $\mathcal Y_{\mathrm{eval}}$ are held fixed when the original and bootstrap cross-Hessians are computed.}\par\textcolor{blue}{At each $y_m$, let $\tau_{\mathrm{boot},m}$ be the bootstrap spectral-norm threshold and set $\tau_m=\max\{\tau_{\mathrm{boot},m},\rho\,\sigma_1(\widehat{\mathbf H}_{A,B}(y_m))\}$. The pointwise estimate $\widehat r_m$ is the number of singular values of $\widehat{\mathbf H}_{A,B}(y_m)$ that exceed $\tau_m$. We report the modal value among $\widehat r_1,\ldots,\widehat r_M$, resolving a tie in favor of the larger rank. All empirical HRC decisions use this aggregated rank.}\par\textcolor{blue}{This procedure is an operational aggregation rule. Theorem~\ref{the:empirical_rank_id} applies when a valid pointwise estimator-error bound is available and does not itself establish bootstrap validity for data-dependent bandwidths or evaluation points. Selecting $h$ and $\mathcal Y_{\mathrm{eval}}$ on an independent pilot sample avoids reusing the estimation sample for these choices. For multiple fixed evaluation points, valid pointwise bounds at level $\eta/M$ can be combined by a union bound. Correspondingly, $q_{\mathrm{boot}}=1-\eta/M$ is used as a practical Bonferroni calibration rather than as a separate bootstrap-consistency guarantee.}\par\subsection{Proof of Theorem~\ref{thm:latent-ci-hrc}}
\label{app:the5}

\begin{theorem}[Latent Conditional Independence via HRC]
Assume a correctly identified pure nonlinear one-factor measurement model whose latent distribution is Markov and faithful to its latent DAG. Let $L_i,L_j$ and $Q$ be as in Theorem~\ref{thm:latent-ci-hrc}, with $|Q|=q$, and construct $A,B$ from pure indicators as specified there. Suppose Assumptions~\ref{assump:regularity}--\ref{assump:nondegeneracy} hold for these groups with respect to their minimum latent separator. Then
\[
L_i\perp\!\!\!\perp L_j\mid Q
\quad\Longleftrightarrow\quad
\operatorname{rank}\!\left(\mathbf H_{A,B}\right)=q.
\]
\end{theorem}

\begin{proof}
Suppose first that $L_i\perp\!\!\!\perp L_j\mid Q$. Under conditional-independence faithfulness of the latent DAG, $Q$ d-separates $L_i$ and $L_j$. Purity of the measurement model then implies that $Q$ d-separates the constructed observed groups $A$ and $B$. To see this, consider any path between the two groups. If either endpoint is a pure indicator of some $L_k\in Q$, the path must pass through $L_k$ as a noncollider and is blocked by conditioning on $Q$. Otherwise, the endpoints are the selected indicators of $L_i$ and $L_j$; removing their two measurement edges gives a path between $L_i$ and $L_j$, which is blocked by $Q$.

Moreover, $Q$ is a minimum latent d-separator for these groups. Indeed, for every $L_k\in Q$, one pure indicator of $L_k$ lies in $A$ and another lies in $B$. The length-two path
\[
X_k^{A}\leftarrow L_k\rightarrow X_k^{B}
\]
is active unless the latent noncollider $L_k$ is included in the separator. Hence every latent d-separator between $A$ and $B$ contains all $q$ members of $Q$. Theorem~\ref{thm:cross-hessian-constraints} therefore yields
\[
\operatorname{rank}\!\left(\mathbf H_{A,B}\right)=q.
\]

Conversely, suppose the cross-Hessian function has rank $q$. Theorem~\ref{thm:cross-hessian-constraints} identifies the minimum latent-separator cardinality as $q$. The pure-indicator paths above force every latent separator to contain all members of $Q$; a separator of cardinality $q$ must therefore be exactly $Q$. Thus $Q$ d-separates $A$ and $B$. If an active path between $L_i$ and $L_j$ given $Q$ existed, adjoining the pure-indicator edges $X_i\leftarrow L_i$ and $L_j\rightarrow X_j$ would produce an active path between $A$ and $B$ given $Q$, a contradiction. Hence $Q$ d-separates $L_i$ and $L_j$. The causal Markov property gives
$L_i\perp\!\!\!\perp L_j\mid Q$.
\end{proof}

\subsection{Proof of Theorem~\ref{thm:hrc-pc-correctness}}
\label{app:the6}
\begin{theorem}[\textbf{Correctness of HRC-PC}]
\textcolor{blue}{Suppose Assumptions~\ref{assump:regularity}--\ref{assump:nondegeneracy} hold for every HRC test performed by HRC-PC, and the latent DAG satisfies the causal Markov and ordinary causal faithfulness assumptions. If HRC-based FOFC correctly recovers the pure measurement clusters and their key latents with the indicators required by HRC-PC, and the empirical HRC rank tests are consistent, then HRC-PC asymptotically recovers the Markov equivalence class of the latent DAG.} 
\end{theorem}
\begin{proof}
By assumption, the first-stage HRC-based FOFC procedure correctly identifies which observed variables share the same key latent, thereby recovering the pure measurement clusters and their associated key latents. Because each relevant cluster contains the required pure indicators, for any latent pair $L_i,L_j$ and conditioning set $Q\subseteq\mathcal L\setminus\{L_i,L_j\}$ considered by HRC-PC, we can construct the corresponding disjoint surrogate groups $A$ and $B$ as specified above.

According to Theorem~\ref{thm:latent-ci-hrc}, the cross-Hessian rank constraint $\mathrm{rank}(\mathbf H_{A,B}) = |Q|$ is strictly equivalent to the latent conditional independence $L_i \perp\!\!\!\perp L_j \mid Q$. By the assumed consistency of the empirical HRC rank test, these rank decisions asymptotically agree with the corresponding latent conditional-independence relations. HRC-PC therefore asymptotically reduces to the oracle PC algorithm on the recovered latent variables.

Under the causal Markov and faithfulness assumptions, the oracle PC algorithm recovers the Markov equivalence class of the latent DAG \citep{spirtes2000causation}. The preceding reduction therefore yields the claimed asymptotic recovery result.
\end{proof}

\section{Further Discussion on Assumptions}

\subsection{Discussion on Pointwise Non-degeneracy}
\label{app:regular_non_degeneracy_condition}
\subsubsection{\textcolor{blue}{Full-Dimensional Conditional Variation}}

\textcolor{blue}{
Fix a minimum-cardinality latent separator $\mathcal Z^*\subseteq\mathcal L$ and, for brevity, write $d=|\mathcal Z^*|$ and $Z\in\mathbb R^d$ for its associated random vector. Assumption~\ref{assump:nondegeneracy}-(i) is a pointwise condition on the conditional distribution of this minimum separator. For $P_{A,B}$-almost every $(a,b)$, it requires finite conditional second moments and
\begin{equation}
\operatorname{Cov}(Z\mid A=a,B=b)\succ 0.
\end{equation}
For every $v\in\mathbb{R}^d$,
\begin{equation}
v^\top\operatorname{Cov}(Z\mid A=a,B=b)v
=
\operatorname{Var}(v^\top Z\mid A=a,B=b).
\end{equation}
Consequently, the covariance is singular if and only if there exist a nonzero $v$ and a scalar $c(a,b)$ such that
$v^\top Z=c(a,b)$ holds $P_{Z\mid A=a,B=b}$-almost surely. Equivalently, the conditional distribution is supported on a proper affine subspace of $\mathbb{R}^d$. This is the precise failure excluded by Assumption~\ref{assump:nondegeneracy}-(i).}

\textcolor{blue}{To illustrate the condition, consider the two-dimensional separator
\begin{equation}
Z=(L_1,L_2)^\top.
\end{equation}
Under conditional independence and Assumption~\ref{assump:local-linearity}, the cross-Hessian factorization is, for $P_{A,B}$-almost every $(a,b)$,
\begin{equation}
\mathbf H_{A,B}
=
\mathbf V_A(a)\,
\operatorname{Cov}(Z\mid A=a,B=b)\,
\mathbf V_B(b)^\top.
\end{equation}
Thus the positive-definiteness condition makes the middle matrix rank two; Assumption~\ref{assump:nondegeneracy}-(ii) then ensures that the two slope matrices do not reduce this rank.
}

\begin{figure}[h]
\centering

\begin{subfigure}{0.48\linewidth}
\centering
\scalebox{0.52}{
\begin{tikzpicture}[
    >={Stealth[length=3.5mm,width=2.4mm]},
    line width=0.9pt,
    latent/.style={
        draw,
        circle,
        minimum size=11mm,
        inner sep=0pt,
        font=\Large\bfseries
    },
    obs/.style={
        draw,
        circle,
        minimum size=11mm,
        inner sep=0pt,
        fill=gray!20,
        font=\Large\bfseries
    },
    groupbox/.style={
        draw,
        rounded corners,
        dashed,
        inner sep=6pt
    },
    edgelabel/.style={
        midway,
        above,
        font=\large,
        fill=white,
        inner sep=1pt
    }
]

% latent variables
\node[latent] (L1) at (0,0) {$L_1$};
\node[latent] (L2) at (6.0,0) {$L_2$};

% observed variables for L1
\node[obs] (X1) at (-2.0,-2.4) {$X_1$};
\node[obs] (X2) at (0,-2.4) {$X_2$};
\node[obs] (X3) at (2.0,-2.4) {$X_3$};

% observed variables for L2
\node[obs] (X4) at (4.0,-2.4) {$X_4$};
\node[obs] (X5) at (6.0,-2.4) {$X_5$};
\node[obs] (X6) at (8.0,-2.4) {$X_6$};

% stochastic latent edge
\draw[->] (L1) -- node[edgelabel] {$L_2=aL_1+\epsilon_2$} (L2);

% latent to observed
\draw[->] (L1) -- (X1);
\draw[->] (L1) -- (X2);
\draw[->] (L1) -- (X4);

\draw[->] (L2) -- (X4);
\draw[->] (L2) -- (X5);
\draw[->] (L2) -- (X6);

% group boxes
\node[
    groupbox,
    fit=(X1)(X2)(X3),
    label={[font=\Large]below:{$A=\{X_1,X_2,X_3\}$}}
] {};

\node[
    groupbox,
    fit=(X4)(X5)(X6),
    label={[font=\Large]below:{$B=\{X_4,X_5,X_6\}$}}
] {};

\end{tikzpicture}
}
\caption{Stochastic linear relation.}
\label{fig:posterior-nondegenerate-stochastic}
\end{subfigure}
\hfill
\begin{subfigure}{0.48\linewidth}
\centering
\scalebox{0.52}{
\begin{tikzpicture}[
    >={Stealth[length=3.5mm,width=2.4mm]},
    line width=0.9pt,
    latent/.style={
        draw,
        circle,
        minimum size=11mm,
        inner sep=0pt,
        font=\Large\bfseries
    },
    obs/.style={
        draw,
        circle,
        minimum size=11mm,
        inner sep=0pt,
        fill=gray!20,
        font=\Large\bfseries
    },
    groupbox/.style={
        draw,
        rounded corners,
        dashed,
        inner sep=6pt
    },
    edgelabel/.style={
        midway,
        above,
        font=\large,
        fill=white,
        inner sep=1pt
    }
]

% latent variables
\node[latent] (L1) at (0,0) {$L_1$};
\node[latent] (L2) at (6.0,0) {$L_2$};

% observed variables for L1
\node[obs] (X1) at (-2.0,-2.4) {$X_1$};
\node[obs] (X2) at (0,-2.4) {$X_2$};
\node[obs] (X3) at (2.0,-2.4) {$X_3$};

% observed variables for L2
\node[obs] (X4) at (4.0,-2.4) {$X_4$};
\node[obs] (X5) at (6.0,-2.4) {$X_5$};
\node[obs] (X6) at (8.0,-2.4) {$X_6$};

% deterministic latent edge
\draw[->] (L1) -- node[edgelabel] {$L_2=aL_1+b$} (L2);

% latent to observed
\draw[->] (L1) -- (X1);
\draw[->] (L1) -- (X2);
\draw[->] (L1) -- (X4);

\draw[->] (L2) -- (X3);
\draw[->] (L2) -- (X5);
\draw[->] (L2) -- (X6);

% group boxes
\node[
    groupbox,
    fit=(X1)(X2)(X3),
    label={[font=\Large]below:{$A=\{X_1,X_2,X_3\}$}}
] {};

\node[
    groupbox,
    fit=(X4)(X5)(X6),
    label={[font=\Large]below:{$B=\{X_4,X_5,X_6\}$}}
] {};

\end{tikzpicture}
}
\caption{Deterministic linear relation.}
\label{fig:posterior-nondegenerate-deterministic}
\end{subfigure}

\vspace{-5pt}
\caption{\textcolor{blue}{Two latent-variable graphs in which $L_1$ and $L_2$ each have children in both observed groups, so the minimum latent d-separator between $A$ and $B$ is $\{L_1,L_2\}$.
In (a), $L_2=aL_1+\epsilon_2$ contains a non-degenerate exogenous noise term; Assumption~\ref{assump:nondegeneracy}-(i) additionally requires $\operatorname{Cov}((L_1,L_2)^\top\mid A=a,B=b)\succ0$ for $P_{A,B}$-almost every $(a,b)$.
In (b), $L_2=aL_1+b$, so the conditional distribution of $(L_1,L_2)^\top$ is supported on the affine line $\{(l_1,l_2):l_2=al_1+b\}$ and its conditional covariance is singular.}}
\label{fig:posterior-nondegenerate}
% \vspace{-15pt}
\end{figure}
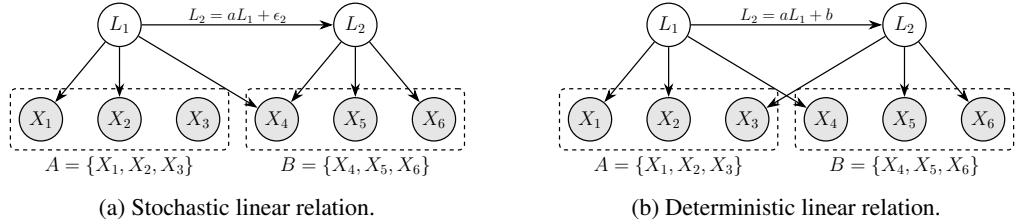

In both graphs of Figure~\ref{fig:posterior-nondegenerate}, the fork paths through $L_1$ and $L_2$ show that neither latent variable alone d-separates $A$ and $B$; hence $\{L_1,L_2\}$ is a minimum latent d-separator. We compare two cases: a stochastic linear relation, where $L_2$ receives its own exogenous noise, and a
deterministic linear relation, where $L_2$ is completely determined by $L_1$.

First, consider the stochastic linear relation
\begin{equation}
    L_2=aL_1+\epsilon_2,
\qquad
\epsilon_2\sim \mathcal N(0,\sigma_2^2),
\qquad
\sigma_2^2>0.
\end{equation}
\textcolor{blue}{
Here $\epsilon_2$ provides an additional exogenous source of variation for $L_2$. Its presence alone is not used as a substitute for Assumption~\ref{assump:nondegeneracy}-(i). The required condition is stated directly: for $P_{A,B}$-almost every $(a,b)$,
\begin{equation}
\operatorname{Cov}\!\left((L_1,L_2)^\top\mid A=a,B=b\right)\succ0,
\qquad
\operatorname{rank}\!\left(
\operatorname{Cov}\!\left((L_1,L_2)^\top\mid A=a,B=b\right)
\right)=2.
\end{equation}
}
If the observation matrices $\mathbf V_A(A)$ and $\mathbf V_B(B)$ both have full column rank two, then
multiplication by these matrices does not reduce the rank. Hence
\begin{equation}
    \mathrm{rank}(\mathbf H_{A,B})
=
\mathrm{rank}\!\left(
\mathbf V_A(A)\,
\operatorname{Cov}(Z\mid A,B)\,
\mathbf V_B(B)^\top
\right)
=
2.
\end{equation}
Therefore, in the stochastic linear case, the cross-Hessian rank correctly recovers the two-dimensional
separating the latent set.

In contrast, consider the deterministic linear relation
\begin{equation}
    L_2=aL_1+b.
\end{equation}
\textcolor{blue}{
Let $v=(-a,1)^\top$. The structural equality implies
\begin{equation}
v^\top Z=L_2-aL_1=b
\end{equation}
under every conditional distribution $P_{Z\mid A=a,B=b}$. Therefore,
\begin{equation}
\operatorname{Var}(v^\top Z\mid A=a,B=b)=0,
\end{equation}
so $\operatorname{Cov}(Z\mid A=a,B=b)$ is singular and the conditional distribution is supported on the proper affine line
$\{(l_1,l_2):l_2=al_1+b\}$.
}
Expanding this covariance gives
\begin{equation}
    \operatorname{Cov}(Z\mid A=a,B=b)
=
\begin{pmatrix}
\operatorname{Var}(L_1\mid A=a,B=b)
&
a\,\operatorname{Var}(L_1\mid A=a,B=b)
\\
a\,\operatorname{Var}(L_1\mid A=a,B=b)
&
a^2\,\operatorname{Var}(L_1\mid A=a,B=b)
\end{pmatrix}.
\end{equation}
Equivalently,
\begin{equation}
    \operatorname{Cov}(Z\mid A=a,B=b)
=
\operatorname{Var}(L_1\mid A=a,B=b)
\begin{pmatrix}
1\\
a
\end{pmatrix}
\begin{pmatrix}
1 & a
\end{pmatrix}.
\end{equation}
This is an outer-product matrix, so whenever $\operatorname{Var}(L_1\mid A=a,B=b)>0$, we have
\begin{equation}
    \mathrm{rank}\!\left(\operatorname{Cov}(Z\mid A=a,B=b)\right)=1.
\end{equation}
Thus, the posterior covariance is not positive definite, and Assumption~\ref{assump:nondegeneracy}-(i) is violated.
Substituting this degenerate posterior covariance into the cross-Hessian decomposition yields
\begin{equation}
    \mathrm{rank}(\mathbf H_{A,B})
=
\mathrm{rank}\!\left(
\mathbf V_A(A)\,
\operatorname{Cov}(Z\mid A=a,B=b)\,
\mathbf V_B(B)^\top
\right)
\leq
\mathrm{rank}\!\left(\operatorname{Cov}(Z\mid A=a,B=b)\right)
=
1.
\end{equation}
\textcolor{blue}{If both slope matrices are nonzero on the rank-one image of the conditional covariance, then the inequality is tight and}
\begin{equation}
    \mathrm{rank}(\mathbf H_{A,B})=1.
\end{equation}

\textcolor{blue}{This example makes the role of Assumption~\ref{assump:nondegeneracy}-(i) explicit. In the stochastic model, the assumption requires the conditional covariance to have rank two for $P_{A,B}$-almost every $(a,b)$. In the deterministic model, the identity $(-a,1)^\top Z=b$ forces the conditional distribution onto a proper affine line, so the conditional covariance has rank one and the cross-Hessian rank is at most one. Thus, the relevant condition is full-dimensional conditional support, stated equivalently through positive-definite conditional covariance.}

\subsubsection{Zero Measurements and Informative Observations}
For a given evaluation point, a nonzero vector $c\in\mathbb R^d$ is a \emph{zero-measurement direction} for group $A$ if $\mathbf V_A(a)c=0$, and analogously for group $B$. In nonlinear models, this direction may vary with the observed value. Assumption~\ref{assump:nondegeneracy}-(ii) excludes such directions for $P_{A,B}$-almost every $(a,b)$. In a linear Gaussian measurement model, the slope matrices are constant and have the same column ranks as the loading matrices. A zero loading column is one example, while linearly dependent nonzero columns can produce the same failure. Conversely, an individual zero loading coefficient does not violate the assumption if the complete loading matrix remains full column rank. When full column rank is feasible, rank-deficient loading matrices form a proper algebraic, and hence Lebesgue-measure-zero, subset of the unrestricted loading-parameter space. This parameter-space statement is distinct from the $P_{A,B}$-almost-everywhere qualification over evaluation points; under structural restrictions that force rank deficiency, the failure need not be measure zero within the restricted model class.

Under conditional independence, the cross-Hessian can be written as
\begin{equation}
\mathbf H_{A,B}
=
\mathbf V_A(A)\,
\operatorname{Cov}(Z\mid A,B)\,
\mathbf V_B(B)^\top .
\end{equation}
Assumption~\ref{assump:nondegeneracy}-(ii) requires both
$\mathbf V_A(a)$ and $\mathbf V_B(b)$ to have full column rank $d$.
Thus, both observed groups must collectively contain information about
every direction of the separator vector $Z\in\mathbb R^d$. If one group
has a zero-measurement direction, then the cross-Hessian rank may be smaller
than $d$ even when the separating latent set is $d$-dimensional and the
posterior covariance is positive definite.

The requirement is imposed on each observed group as a whole; it does not
require every observed variable to depend on every latent variable. For
example, when $Z=(L_1,L_2)^\top$, one variable in group $A$ may primarily
measure $L_1$ and another may primarily measure $L_2$, provided that the
group jointly spans both latent directions. What is excluded is a group-level
loss of a latent direction.

Figure~\ref{fig:informative-observation} illustrates the distinction. In
both graphs, the separating latent set is $\mathcal Z=\{L_1,L_2\}$. In
Figure~\ref{fig:informative-observation-good}, both observed groups contain
information about two independent latent directions. In
Figure~\ref{fig:informative-observation-bad}, both latents can have nonzero
effects on group $B$, but the two loading columns are proportional. Hence
$B$ responds only to one linear combination of $(L_1,L_2)$ and has a
zero-measurement direction, so its cross-Hessian cannot recover both
graphically present latent directions.

\begin{figure}[h]
\centering

\begin{subfigure}{0.48\linewidth}
\centering
\scalebox{0.50}{
\begin{tikzpicture}[
    >={Stealth[length=3.5mm,width=2.4mm]},
    line width=0.9pt,
    latent/.style={
        draw,
        circle,
        minimum size=11mm,
        inner sep=0pt,
        font=\Large\bfseries
    },
    obs/.style={
        draw,
        circle,
        minimum size=11mm,
        inner sep=0pt,
        fill=gray!20,
        font=\Large\bfseries
    },
    groupbox/.style={
        draw,
        rounded corners,
        dashed,
        inner sep=6pt
    }
]

% latent variables
\node[latent] (L1) at (0,0) {$L_1$};
\node[latent] (L2) at (7.2,0) {$L_2$};

% observed variables in A
\node[obs] (X1) at (-2.4,-2.4) {$X_1$};
\node[obs] (X2) at (0,-2.4) {$X_2$};
\node[obs] (X3) at (2.4,-2.4) {$X_3$};

% observed variables in B
\node[obs] (X4) at (4.8,-2.4) {$X_4$};
\node[obs] (X5) at (7.2,-2.4) {$X_5$};
\node[obs] (X6) at (9.6,-2.4) {$X_6$};

% latent to observed: A informative about both
\draw[->] (L1) -- (X1);
\draw[->] (L1) -- (X2);
\draw[->] (L1) -- (X4);
\draw[->] (L2) -- (X3);

% latent to observed: B informative about both
\draw[->] (L1) -- (X4);
\draw[->] (L2) -- (X5);
\draw[->] (L1) -- (X6);
\draw[->] (L2) -- (X6);

% group boxes
\node[
    groupbox,
    fit=(X1)(X2)(X3),
    label={[font=\Large]below:{$A=\{X_1,X_2,X_3\}$}}
] {};

\node[
    groupbox,
    fit=(X4)(X5)(X6),
    label={[font=\Large]below:{$B=\{X_4,X_5,X_6\}$}}
] {};

\end{tikzpicture}
}
\caption{Both groups are informative.}
\label{fig:informative-observation-good}
\end{subfigure}
\hfill
\begin{subfigure}{0.48\linewidth}
\centering
\scalebox{0.50}{
\begin{tikzpicture}[
    >={Stealth[length=3.5mm,width=2.4mm]},
    line width=0.9pt,
    latent/.style={
        draw,
        circle,
        minimum size=11mm,
        inner sep=0pt,
        font=\Large\bfseries
    },
    obs/.style={
        draw,
        circle,
        minimum size=11mm,
        inner sep=0pt,
        fill=gray!20,
        font=\Large\bfseries
    },
    groupbox/.style={
        draw,
        rounded corners,
        dashed,
        inner sep=6pt
    }
]

% latent variables
\node[latent] (L1) at (0,0) {$L_1$};
\node[latent] (L2) at (7.2,0) {$L_2$};

% observed variables in A
\node[obs] (X1) at (-2.4,-2.4) {$X_1$};
\node[obs] (X2) at (0,-2.4) {$X_2$};
\node[obs] (X3) at (2.4,-2.4) {$X_3$};

% observed variables in B
\node[obs] (X4) at (4.8,-2.4) {$X_4$};
\node[obs] (X5) at (7.2,-2.4) {$X_5$};
\node[obs] (X6) at (9.6,-2.4) {$X_6$};

% latent to observed: A informative about both
\draw[->] (L1) -- (X1);
\draw[->] (L1) -- (X2);
\draw[->] (L2) -- (X1);
\draw[->] (L2) -- (X3);

% latent to observed: both latents affect B, but the two loading
% columns are proportional, leaving one latent direction unmeasured
\draw[->] (L1) -- (X4);
\draw[->] (L2) -- (X3);
\draw[->] (L1) -- (X5);
\draw[->] (L2) -- (X5);
\draw[->] (L1) -- (X6);
\draw[->] (L2) -- (X6);

% group boxes
\node[
    groupbox,
    fit=(X1)(X2)(X3),
    label={[font=\Large]below:{$A=\{X_1,X_2,X_3\}$}}
] {};

\node[
    groupbox,
    fit=(X4)(X5)(X6),
    label={[font=\Large]below:{$B=\{X_4,X_5,X_6\}$}}
] {};

\end{tikzpicture}
}
\caption{Group $B$ measures only one latent direction.}
\label{fig:informative-observation-bad}
\end{subfigure}

% \vspace{-5pt}
\caption{Illustration of the informative-observation condition. In both graphs,
the separating latent set is $\mathcal Z=\{L_1,L_2\}$. In (a), both
observed groups collectively span two latent directions. In (b), both
latents have nonzero effects on group $B$, but their loading columns are
proportional. Thus $B$ measures only one linear combination of the latents,
$\mathbf V_B(b)$ has rank one, and the cross-Hessian loses one graphically
present latent direction.}
\label{fig:informative-observation}
% \vspace{-10pt}
\end{figure}
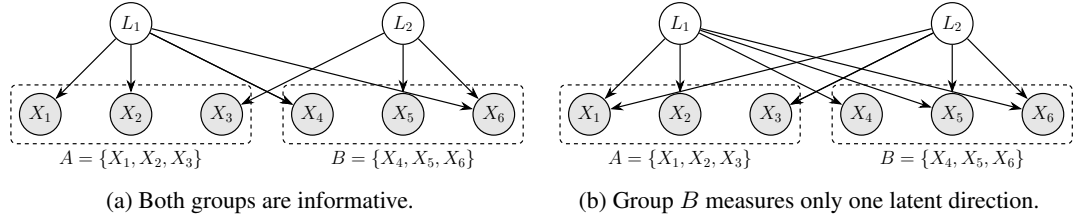

To make this point explicit, consider the special case of a linear Gaussian measurement model:
\begin{equation}
    A=\Lambda_A Z +\epsilon_A,
\qquad
B=\Lambda_B Z+\epsilon_B,
\end{equation}
where $Z=(L_1,L_2)^\top$, $\epsilon_A\sim \mathcal N(0,\Omega_A)$, and
$\epsilon_B\sim \mathcal N(0,\Omega_B)$. Then
\begin{equation}
    \frac{\partial \ln \textcolor{blue}{p_{A\mid Z}(A\mid Z)}}{\partial A}
=
-\Omega_A^{-1}A+\Omega_A^{-1}\Lambda_A Z,
\qquad
\frac{\partial \ln \textcolor{blue}{p_{B\mid Z}(B\mid Z)}}{\partial B}
=
-\Omega_B^{-1}B+\Omega_B^{-1}\Lambda_B Z.
\end{equation}
Thus, in this special case,
\begin{equation}
    \mathbf V_A(A)=\Omega_A^{-1}\Lambda_A,
\qquad
\mathbf V_B(B)=\Omega_B^{-1}\Lambda_B.
\end{equation}
Since $\Omega_A$ and $\Omega_B$ are positive definite noise covariance matrices, multiplication by
$\Omega_A^{-1}$ or $\Omega_B^{-1}$ does not change column rank. Therefore,
\begin{equation}
    \mathrm{rank}(\mathbf V_A(A))=\mathrm{rank}(\Lambda_A),
\qquad
\mathrm{rank}(\mathbf V_B(B))=\mathrm{rank}(\Lambda_B).
\end{equation}

In the informative case, suppose, for example, that
\begin{equation}
    \Lambda_A
=
\begin{pmatrix}
\varepsilon_1 & 0\\
0 & \varepsilon_2\\
\varepsilon_3 & \varepsilon_4
\end{pmatrix},
\qquad
\Lambda_B
=
\begin{pmatrix}
\beta_1 & 0\\
0 & \beta_2\\
\beta_3 & \beta_4
\end{pmatrix},
\end{equation}
where the columns of both matrices are linearly independent. Then
\begin{equation}
    \mathrm{rank}(\mathbf V_A)=
\mathrm{rank}(\mathbf V_B)=2.
\end{equation}
If Assumption~\ref{assump:nondegeneracy}-(i) also holds, so that
\begin{equation}
    \mathrm{rank}\!\left(\operatorname{Cov}(Z\mid A,B)\right)=2\quad P_{A,B}\text{-almost surely},
\end{equation}
then
\begin{equation}
    \mathrm{rank}(\mathbf H_{A,B})
=
\mathrm{rank}\!\left(
\mathbf V_A\,
\operatorname{Cov}(Z\mid A,B)\,
\mathbf V_B^\top
\right)
=
2.
\end{equation}
Hence, the cross-Hessian rank correctly recovers the two-dimensional separating latent set.

In contrast, suppose the two loading columns for group $B$ are
proportional:
\begin{equation}
\Lambda_B
=
\begin{pmatrix}
\beta_1 & c\beta_1\\
\beta_2 & c\beta_2\\
\beta_3 & c\beta_3
\end{pmatrix},
\qquad c\neq 0.
\end{equation}
Both latent variables may therefore have nonzero arrows into every variable
in $B$, while $B$ depends only on the combination $L_1+cL_2$. Indeed,
$\Lambda_B(-c,1)^\top=0$, and hence
\begin{equation}
\operatorname{rank}(\mathbf V_B)
=
\operatorname{rank}(\Lambda_B)
=
1.
\end{equation}
Even if group $A$ is fully informative and the posterior covariance of $Z$
is positive definite,
\begin{equation}
\operatorname{rank}(\mathbf H_{A,B})
\leq
\operatorname{rank}(\mathbf V_B)
=
1.
\end{equation}
Thus, although the graph contains two latent paths and its minimum latent
d-separating set has cardinality two, the observed cross-Hessian sees only
one direction. This is precisely the rank loss excluded by the full-column-rank clause of Assumption~\ref{assump:nondegeneracy}-(ii).
A literal
zero loading column is the boundary case $c=0$; the proportional-column
construction shows that zero measurement refers more generally to an
unobserved latent direction, not necessarily to an individually absent
latent variable.

\subsection{Exact Affinity of Conditional Derivatives}
\label{app:local_linearity}

\textcolor{blue}{
Assumption~\ref{assump:local-linearity} is imposed relative to a specified
latent separator $\mathcal Z\subseteq\mathcal L$. Writing
$d=|\mathcal Z|$ and $Z\in\mathbb R^d$ for its associated random vector,
the assumption requires the conditional log-density derivatives to be
exactly affine in the realization $z$. It is not a global linearity
condition on the structural mechanisms, nor is it required to hold
simultaneously for every subset of $\mathcal L$.}

\textcolor{blue}{The condition permits genuinely nonlinear measurement functions for
arbitrary separator dimension. Let $X\in(0,\infty)^m$ satisfy
\[
X=\exp(\Lambda Z+\epsilon),
\qquad
\epsilon\sim\mathcal N(0,\Omega),
\]
where the exponential is applied componentwise,
$\Lambda\in\mathbb R^{m\times d}$, and $\Omega\succ0$. For a realization
$x$, let $D(x):=\operatorname{diag}(x)$. The conditional log-density is
\[
\log p_{X\mid Z}(x\mid z)
=
-\frac12(\log x-\Lambda z)^\top
\Omega^{-1}(\log x-\Lambda z)
-\sum_{i=1}^m\log x_i+C,
\]
where $\log x$ is componentwise. Differentiation gives
\[
\nabla_x\log p_{X\mid Z}(x\mid z)
=
\underbrace{-D(x)^{-1}\Omega^{-1}\log x-D(x)^{-1}\mathbf 1}_{\mathbf u_X(x)}
+
\underbrace{D(x)^{-1}\Omega^{-1}\Lambda}_{\mathbf V_X(x)}z.
\]
Thus the derivative is exactly affine in the full $d$-dimensional
separator realization, even though the measurement function is nonlinear.
Applying this construction separately to $A$ and $B$ yields
Assumption~\ref{assump:local-linearity}. Figure~\ref{fig:main_nonlinear_density_local_linearity}
is the scalar special case $m=d=1$, $\Lambda=1$, and $\Omega=1$.
This example gives a concrete sufficient nonlinear class; it does not claim
that every post-nonlinear model satisfies the assumption.
}

\textcolor{blue}{
\paragraph{An iterated post-nonlinear special case.}
The same affine-derivative property can hold after several nonlinear
structural layers. Set $W_0=Z$ and $g_0$ to the identity map, and consider
\[
W_k
=
g_k\!\left(
M_k g_{k-1}^{-1}(W_{k-1})+\epsilon_k
\right),
\qquad k=1,\ldots,K,
\]
where each $g_k$ is a twice continuously differentiable diffeomorphism and
the noises $\epsilon_k\sim\mathcal N(0,\Omega_k)$ are mutually
independent and independent of $Z$. The intermediate variables
$W_1,\ldots,W_{K-1}$ may be latent, while $X:=W_K$ is observed. Defining
$U_k:=g_k^{-1}(W_k)$ gives the linear Gaussian recursion
$U_k=M_kU_{k-1}+\epsilon_k$. Consequently, after marginalizing the
intermediate variables,
\[
U_K\mid Z=z
\sim
\mathcal N(\Lambda z,\Omega),
\qquad
\Lambda=M_K\cdots M_1,
\]
where $\Omega\succ0$ is the covariance obtained by propagating the layer
noises. Let $h:=g_K^{-1}$. A change of variables yields
\[
\begin{aligned}
\nabla_x\log p_{X\mid Z}(x\mid z)
&=
\underbrace{
\nabla_x\log|\det J_h(x)|
-
J_h(x)^\top\Omega^{-1}h(x)
}_{\mathbf u_X(x)}
+
\underbrace{
J_h(x)^\top\Omega^{-1}\Lambda
}_{\mathbf V_X(x)}z.
\end{aligned}
\]
Thus the conditional log-density derivative remains exactly affine in
the full separator realization $z$, although the structural relations are
iterated and nonlinear and the intermediate variables have been
marginalized. Applying two such branches with conditionally independent
noises to the groups $A$ and $B$ gives
Assumption~\ref{assump:local-linearity}.
}

\section{Details of FOFC Framework}
\label{app:fofc_details}

Find One Factor Clusters (FOFC) is a clustering framework for learning pure one-factor measurement
clusters from observed indicators. The goal is to group observed variables that are likely to share the same
latent parent. FOFC does not assume that the number of latent variables or the cluster assignments are known
in advance. Instead, it searches for sets of observed variables whose covariance structure is consistent with a
one-factor measurement model.

The algorithm is built on quartet tests. For a set of four observed variables, FOFC checks whether the quartet
satisfies the rank pattern expected from a one-factor structure, which is usually implemented through
vanishing tetrad constraints. A triple of variables is retained as a candidate pure triple if adding any other
observed variable forms a valid one-factor quartet. These candidate triples are then merged into larger
clusters, and the final output is obtained by selecting maximal non-overlapping clusters.

\begin{algorithm}[h]
\caption{Find One Factor Clusters (FOFC)}
\label{alg:fofc}
\begin{algorithmic}[1]
\REQUIRE Observed variables $\mathcal X$, quartet test
\ENSURE Recovered measurement clusters $\mathcal C$

\STATE $\mathcal T \leftarrow \emptyset$

\FOR{each triple $S\subset\mathcal X$ with $|S|=3$}
    \STATE $\mathrm{Pure} \leftarrow \mathrm{TRUE}$
    \FOR{each variable $X_i\in\mathcal X\setminus S$}
        \IF{\textsc{QuartetTest}$(S\cup\{X_i\})=\mathrm{FALSE}$}
            \STATE $\mathrm{Pure} \leftarrow \mathrm{FALSE}$
            \STATE \textbf{break}
        \ENDIF
    \ENDFOR
    \IF{$\mathrm{Pure}=\mathrm{TRUE}$}
        \STATE $\mathcal T \leftarrow \mathcal T\cup\{S\}$
    \ENDIF
\ENDFOR

\STATE $\mathcal C_{\mathrm{cand}} \leftarrow \mathcal T$

\FOR{each candidate cluster $C\in \mathcal C_{\mathrm{cand}}$}
    \FOR{each variable $X_i\in\mathcal X\setminus C$}
        \STATE $\mathrm{Expandable} \leftarrow \mathrm{TRUE}$
        \FOR{each pair $P\subset C$ with $|P|=2$}
            \IF{$P\cup\{X_i\}\notin \mathcal T$}
                \STATE $\mathrm{Expandable} \leftarrow \mathrm{FALSE}$
                \STATE \textbf{break}
            \ENDIF
        \ENDFOR
        \IF{$\mathrm{Expandable}=\mathrm{TRUE}$}
            \STATE $C \leftarrow C\cup\{X_i\}$
        \ENDIF
    \ENDFOR
\ENDFOR

\STATE $\mathcal C \leftarrow \emptyset$

\WHILE{$\mathcal C_{\mathrm{cand}}\neq \emptyset$}
    \STATE Select the largest cluster $C^\star$ in $\mathcal C_{\mathrm{cand}}$
    \STATE $\mathcal C \leftarrow \mathcal C\cup\{C^\star\}$
    \STATE Remove from $\mathcal C_{\mathrm{cand}}$ all clusters that intersect $C^\star$
\ENDWHILE

\RETURN $\mathcal C$
\end{algorithmic}
\end{algorithm}

Here, \textsc{QuartetTest}$(Q)$ denotes the statistical test used to decide whether a quartet $Q$ is consistent
with a one-factor measurement structure. In the original FOFC algorithm, this is based on vanishing tetrad
constraints. The first stage finds candidate pure triples, the second stage grows these triples into larger
candidate clusters, and the final stage selects maximal non-overlapping clusters as the recovered measurement
model.

\section{Experiment}

\subsection{Synthetic Dataset}
\label{app:synthetic_data}
The synthetic data were generated from a nonlinear structural equation model with latent measurement structure. For each setting, we fixed the number of latent variables and the number of observed measurements per latent variable \(M\). We first generated a directed acyclic graph over the latent variables by sampling a random topological ordering, constructing a connected directed tree, and then optionally adding extra directed edges. Each latent edge coefficient was sampled independently from \(\mathrm{Unif}(0.7,1.3)\). Root latent variables were sampled from \(\mathrm{Unif}(a,b)\), while non-root latent variables were generated in topological order as

\begin{equation}
    L_j=\sum_{i\in \mathrm{Pa}(j)} \beta_{ij} L_i+\epsilon_j,\qquad 
\epsilon_j\sim \mathcal N(0,0.2^2).
\end{equation}
Each latent variable \(L_j\) then generated \(M\) pure observed measurements:

\begin{equation}
    X_{jm}=b_{jm} f_{jm}\big((m+1)L_j\big)+\eta_{jm},\qquad
\eta_{jm}\sim \mathcal N(0,\sigma_x^2).
\end{equation}
The loading coefficient \(b_{jm}\) had a random sign and magnitude sampled from \(\mathrm{Unif}(2,5)\). The nonlinear measurement function \(f_{jm}\) was chosen from \(\{\sin,\sinh,\tanh,\mathrm{softplus}\}\). In the single-function setting, all measurements in a dataset used the same nonlinear function; in the mixed setting, each measurement was assigned a function independently. The observation noise level \(\sigma_x\) was selected from a predefined grid. All datasets were generated with fixed random seeds for reproducibility, and each observed variable was standardized before evaluation. All experiments are run on nodes with 128 logical CPU cores.\par\textcolor{blue}{We consider four single-function settings, using $\sin$, $\sinh$, $\tanh$, or softplus for all measurements, and a mixed setting in which each measurement independently selects one of these functions. The number of latent variables is varied over $\{4,6,8,10\}$. For each dimension, we generate 10 random latent graphs and run 10 random seeds per graph. We use sample sizes $N\in\{200,1000,2000,\ldots,9000\}$. Latent-variable location is evaluated by F1, while latent causal discovery is evaluated by F1 and structural Hamming distance (SHD) for the latent Markov equivalence class. We report the mean and standard deviation over all graph--seed combinations.}

\subsection{Real-world Dataset}
\label{app:real_data}
The real-world dataset used in this study is based on the August 2025 vintage of FRED-MD, a monthly U.S. macroeconomic database introduced by McCracken and Ng (2015) as a publicly accessible and regularly updated resource for large-scale empirical macroeconomic analysis. The raw file contains monthly macroeconomic and financial indicators, together with the transformation code row provided by FRED-MD. After removing this metadata row, the sample spans January 1959 to July 2025. For our real-data experiment, we focus on a financially interpretable subset of 14 variables organized into four latent macro-financial groups: interest rates, credit conditions, foreign exchange, and stock-market/volatility conditions. The interest-rate group includes the federal funds rate, short-term Treasury bill rates, medium- and long-term Treasury yields, and the term spread; the credit group includes Baa corporate bond yields and credit spreads; the foreign-exchange group includes a trade-weighted dollar index and bilateral exchange-rate proxies; and the stock/volatility group includes the S\&P 500, a valuation proxy, and a volatility proxy. When a requested series was unavailable in the current FRED-MD vintage, we used the closest available proxy or constructed a spread from available series. The final analysis used the common non-missing sample of the selected variables, applied a first-difference preprocessing step, and evaluated the learned latent CPDAG against the domain-specified macro-financial structure.

\subsection{More Experiment Results}
\label{app:more_exp}
\begin{table}[h]
\centering
\caption{Sensitivity to noise level on L4M3 with $n=3000$.}
\label{tab:noise_sensitivity_l4m3}
\begin{tabular}{c|c|cc}
\toprule
Noise & Cluster F1 & Structure F1 & Structure SHD \\
\midrule
1.00 & $0.7143 \pm 0.2771$ & $0.5430 \pm 0.3031$ & $2.0000 \pm 1.1466$ \\
0.75 & $0.8663 \pm 0.2160$ & $0.7277 \pm 0.2291$ & $1.3960 \pm 1.0341$ \\
0.50 & $0.9387 \pm 0.1475$ & $0.8469 \pm 0.1764$ & $0.9540 \pm 1.0593$ \\
0.25 & $0.9910 \pm 0.0556$ & $0.9502 \pm 0.1088$ & $0.2820 \pm 0.5789$ \\
0.15 & $0.9993 \pm 0.0149$ & $0.9643 \pm 0.0991$ & $0.1900 \pm 0.4923$ \\
\bottomrule
\end{tabular}
\end{table}

\begin{table}[h]
\centering
\caption{Sensitivity to latent variable range with $n=3000$ and noise level fixed at $0.15$.}
\label{tab:range_sensitivity_l4m3}
\begin{tabular}{c|c|cc}
\toprule
Range & Cluster F1 & Structure F1 & Structure SHD \\
\midrule
$[-0.5, 0.5]$ & $0.9993 \pm 0.0149$ & $0.9643 \pm 0.0991$ & $0.1900 \pm 0.4923$ \\
$[-1, 1]$ & $0.9452 \pm 0.1320$ & $0.8268 \pm 0.2060$ & $0.8780 \pm 0.9492$ \\
$[-1.5, 1.5]$ & $0.7543 \pm 0.2451$ & $0.6064 \pm 0.2661$ & $1.8280 \pm 1.0792$ \\
$[-2, 2]$ & $0.6025 \pm 0.2382$ & $0.5133 \pm 0.2901$ & $2.1960 \pm 1.2201$ \\
$[-2.5, 2.5]$ & $0.4990 \pm 0.2203$ & $0.4348 \pm 0.3102$ & $2.3980 \pm 1.1757$ \\
\bottomrule
\end{tabular}
\end{table}

Assumption~\ref{assump:local-linearity} is an exact condition, whereas the nonlinear measurement functions used here are not constructed to satisfy it exactly. Theorem~\ref{the:error_bound} instead bounds the resulting structural approximation error in terms of posterior concentration and conditional-score curvature. The sensitivity study probes related regimes by varying the observation-noise level and latent range. Tables~\ref{tab:noise_sensitivity_l4m3} and~\ref{tab:range_sensitivity_l4m3} show improved empirical performance at lower noise levels and over narrower latent ranges. These results provide robustness evidence beyond the exact affine-derivative setting, but they do not directly verify the numerical bound in Theorem~\ref{the:error_bound}.

\section{Broader Impact}
\label{app:border}

This work may help improve the interpretability of latent-variable models in domains where observed variables are noisy nonlinear measurements of hidden factors, such as economics, finance, health care, and scientific discovery. By using cross-Hessian rank constraints to identify latent dimensions and latent causal structure, the proposed method can provide more structured explanations of high-dimensional observations. However, in sensitive applications, the recovered graph should be viewed as exploratory evidence rather than definitive causal knowledge, and should be interpreted together with domain expertise and data-quality considerations.

\section{Limitations}
\label{app:limit}

The current application to latent-variable location and causal discovery follows the pure measurement model setting, where observed indicators are assumed to be associated with latent variables in a structured way. More general measurement structures, such as mixed indicators, correlated measurement errors, or complex observed causal relations, may require additional algorithmic development. The exact HRC characterization also relies on the affine-derivative and non-degeneracy conditions; these are explicit sufficient conditions and are not claimed to hold for arbitrary nonlinear structural causal models. Finally, the HRC-PC consistency result is conditional on correct first-stage clustering and an available consistent cross-Hessian rank test. The fixed finite-sample bootstrap calibration used in the experiments is not itself proved to be rank consistent.
% \section{Limitation}

\clearpage

}

\end{document}